%% file: arxiv_vs.tex
\documentclass[11pt,letterpaper]{article} 
\usepackage[letterpaper,top=1in,bottom=1in,left=1in,right=1in]{geometry} 
\usepackage[hyphens]{url} 
\usepackage{algorithm}
\usepackage{algpseudocode}

\usepackage{amsmath,amssymb,amsthm}
\usepackage{mathtools}
\usepackage{bbm}          % \mathbbm{1}
\usepackage{booktabs}
\usepackage{microtype}
\usepackage{graphicx}
\usepackage{subcaption}
\usepackage{hyperref}
\usepackage{cleveref}
\usepackage[dvipsnames]{xcolor}
\usepackage[normalem]{ulem}
\usepackage{natbib}  % DO NOT CHANGE THIS AND DO NOT ADD ANY OPTIONS TO IT

\newtheorem{theorem}{Theorem}

\newtheorem{proposition}[theorem]{Proposition}

\newtheorem{definition}[theorem]{Definition}
\newtheorem{assumption}[theorem]{Assumption}

\newcommand{\Rtot}{R_{\mathrm{total}}}    % total regret

\newcommand{\E}{\mathbb{E}}

\title{Fair Recommender Systems for Intersecting Groups} 

\author{ Christine Ling\thanks{Georgia Institute of Technology. Email: cling39@gatech.edu} \and Juba Ziani\thanks{Georgia Institute of Technology. Email: jziani3@gatech.edu} } \date{\today}

\begin{document}

\maketitle

%% Authors omitted for anonymous review

\begin{abstract}
Recommendation systems that optimize for the average user can fail to capture the preferences of diverse demographic groups. We consider a dynamic recommendation system, where arriving users can belong to different demographic groups. Importantly, these groups can intersect, and a user may belong to several groups. We first show that over time, any group-blind policy (which excludes all current-user group information) recommending $K \geq 1$ items incurs linear total regret under mild conditions. In turn, we introduce three group-aware algorithms based on the classical Upper Confidence Bound algorithm (UCB): i) \emph{Intersection-Fitting}, which runs UCB on each intersection, but has a sublinear group- and intersection-level regret guarantee with an exponential dependence in the number of groups; and ii) two group-level heuristics, \emph{Inner-Outer UCB} and \emph{Group-Item UCB} which allocate different arms to different groups, and use a sleeping expert meta-learner to choose which active group's recommendation an intersection should inherit at any given time step. We provide experimental results on two datasets, Goodreads and BoardGameGeek, and highlight that our heuristics can be preferable to the intersection-fitting algorithm when intersection-level preferences are closely related to group-level preferences, but that intersection-fitting performs better when different intersections within a same group prefer different items.
%We study several algorithms to address this. Intersection-fitting has a formal regret guarantee, but needs to fit an exponential number of intersections and may be suboptimal when intersections inherit group structure. To address this, we adapt two optimistic, Upper Confidence Bound (UCB)-based heuristics that make recommendations based on group-level structure. While these heuristics lack formal regret guarantees, our experiments show that pooling data generally outperforms intersection-fitting when intersections inherit group structure (Goodreads). Conversely, on BoardGameGeek, pooling degrades performance, and exact intersection-fitting achieves lower regret. This contrast is consistent with a trade-off between sharing information across overlapping groups and learning each intersection independently.
\end{abstract}

\maketitle
% \raggedbottom
% \clubpenalty=150
% \widowpenalty=150
% \brokenpenalty=100

\section{Introduction}
\label{sec:intro}
Sequential recommendation systems struggle to serve user populations with heterogeneous tastes. When a group-blind algorithm %(which excludes all current-user information)
optimizes for the average user, it marginalizes niche preferences. The best such a platform can do is serve a "one-size-fits-all'' slate, which can ignore preferences of and result in smaller welfare for smaller groups. Closing this gap requires conditioning recommendations on user demographics. 

Group-aware recommendations are straightforward when groups are disjoint. However, real-world users often belong to multiple, overlapping groups simultaneously (e.g., age, region, genre preference). This intersecting structure results in a fundamental learning challenge: there can be too many overlapping combinations to learn efficiently. Maintaining an independent policy for every demographic intersection requires a state space exponential in the number of base groups $G$, i.e. of size up to $2^G$. This fragments the exploration and learning budget and risks slow finite-horizon learning on sparse intersections, even when the original base groups are well-represented. Our paper aim to address this gap by proposing recommendation algorithms where intersection-level recommendations inherit the preferences learned at the level of the groups that constitute these intersections.

\paragraph{Summary of contributions.} Our contributions are as follows:
\begin{itemize}
    \item Section~\ref{sec:model} introduces a general model of a welfare-maximizing recommender serving users who belong to arbitrarily many, possibly overlapping, demographic groups.
    \item Section~\ref{sec:group_blind} proves that every group-blind policy recommending $K$-slates incurs strictly linear total regret whenever the population-optimal slate is strictly suboptimal for a demographic segment (Proposition~\ref{prop:universal}), and gives an instance with an explicit constant per-round gap (Proposition~\ref{prop:constantgap}).
\item Section~\ref{sec:groupaware} introduces group-aware algorithms. We prove a worst-case $\Omega(\sqrt{\alpha_kT})$ regret lower bound for any algorithm. We then study three approaches: exact intersection-fitting (which guarantees no-intersection regret), and two pooling heuristics (Inner--Outer UCB and Group--Item UCB). The pooling methods require only $O((G+1)n)$ item statistics and share data across overlapping attributes.
    % \item Section~\ref{sec:experiments} compares these methods empirically on rating simulations built from Goodreads and BoardGameGeek data. These evaluations are consistent with a core trade-off in intersectional learning: the sample starvation of fitting each exact demographic intersection independently versus the risk of cross-group interference when pooling data across base groups.
    \item Section~\ref{sec:experiments} empirically compares these methods on Goodreads and BoardGameGeek rating simulations. The results are consistent with a core trade-off in intersectional learning: the sample starvation of fitting each exact demographic intersection independently versus the risk of cross-group interference of pooling data.
\end{itemize}

\section{Related Work}
\label{sec:related}
\noindent \textbf{Fairness and Diversity in Recommender Systems:}
Prior work shows that collaborative filtering and popularity-based recommendation algorithms narrow consumption diversity \cite{fleder2007recommenders, chaney2018algorithmic} and exacerbate exposure disparities \cite{mansoury2020feedback, celis2019controlling}. In sequential settings, feedback dynamics induce degenerate recommendation cycles where historically unobserved items are starved of exposure \cite{jiang2019degenerate, sinha2017deconvolving}. %Our model relates to performative prediction, where algorithmic choices shift the data distribution \cite{perdomo2020performative, mendler2020stochastic}, and to strategic content adaptation under algorithmic scoring \cite{hardt2015strategicclassification, jagadeesan2023supply}.
While recent recommendation literature studies unequal quality across multiple user and item group combinations on existing datasets \cite{intersectional_www_2024}, we analyze this as a sequential group-regret problem. 

\smallskip
\noindent \textbf{Fair Online Learning and Constrained Online Learning:} Foundational works introduce meritocratic fairness, requiring that an algorithm not favor an arm over another with higher expected reward \cite{joseph2016fair, joseph2018meritocratic}. Related work extends this principle to reinforcement learning, comparing actions through their long-term discounted rewards \cite{jabbari2017fairness}. Other approaches study calibrated fairness, matching each arm’s selection probability to its probability of having the highest realized reward \cite{liu2017calibrated}, or impose minimum arm-selection or user-allocation rates \cite{patil2020achieving, chen2020fair}. Offline contextual-bandit methods accommodate a broad family of fairness constraints with high-probability guarantees \cite{metevier2019offline}, while constrained ranking methods impose representation requirements on groups of items within ranking prefixes \cite{celis2018ranking}. 
While some recent literature handles overlapping groups, our framework emphasizes the specific combination of group welfare regret, slate recommendations, and selected-item feedback.
% However, much of this literature models user or item categories as strictly disjoint; in contrast, our framework explicitly models overlapping, intersectional memberships.
Most closely related is \cite{acharya2023oracle}, who develop oracle-efficient algorithms for achieving multi-group regret in overlapping demographic settings. Their reduction coordinates group learners using an observed outcome that evaluates each learner's loss. 

\smallskip
\noindent \textbf{Intersectional Fairness:} \cite{kearns2018preventing} show that auditing fairness along isolated marginal categories allows discrimination across intersecting subgroups (``fairness gerrymandering''). This motivated subgroup fairness and multicalibration frameworks in supervised learning, which impose fairness or calibration requirements across overlapping demographic collections \cite{hebert2018multicalibration, kim2018multiaccuracy, rothblum2021multigroupagnosticpaclearnability, tosh2022multigroup, foulds2019intersectional}.

\section{Model}
\label{sec:model}
We model a welfare-maximizing recommendation platform serving users who arrive over $T$ discrete rounds, indexed $t \in \{0, 1, \dots, T-1\}$. Users belong to potentially many, overlapping demographic groups simultaneously. Group-aware policies observe the arriving user's active base-group set before recommending items; group-blind policies ignore it. The full attribute vector need not be observed.

% \paragraph{Fairness metric.} We evaluate fairness directly through regret. We measure group fairness as similar and sublinear regrets across all base demographic groups, ensuring no single group pays a linear welfare penalty to subsidize the population average. Intersectional fairness is measured as similar regret across exact overlapping intersections, ensuring that strong performance on a broad base group does not hide the sample starvation of sparse intersections within it.

\noindent \paragraph{Users and Group Membership}
\label{sec:model-users-valuations}
Each arriving user $i$ has an attribute vector $a_i\in\mathcal A$, drawn i.i.d.\ from a population distribution $\mathcal D$; $i_t$ is the user arriving at round $t$. To avoid an exponential parameterization, we describe demographic structure through a finite collection of $G$ base groups $g_1,\dots,g_G$ rather than the full attribute vector. An indicator $\phi_k$ determines whether a user belongs to $g_k$; if $\phi_k(a_i)=1$ we say $g_k$ is active (or \emph{awake}) for user $i$. We also introduce a universal, always-active group $g_0$ ($\phi_0\equiv1$) to ensure group-aware algorithms can match group-blind performance when group-blindness is optimal. Identifying $g_k$ with its index $k$, user $i$'s active set is $\mathcal G_i:=\mathcal G(a_i)=\{k:\phi_k(a_i)=1\}$. The arrival rate of $g_k$ is $\alpha_k:=\Pr_{a\sim\mathcal D}(\phi_k(a)=1)$ and $\mathcal D_k:=\mathcal D(\cdot\mid\phi_k(a)=1)$.

\noindent \paragraph{Group intersections}
For $P\subseteq\{0,\dots,G\}$, the exact intersection $\iota_P:=\{a:\mathcal G(a)=P\}$ contains users whose active set is exactly $P$, including non-membership in all groups outside $P$; its distribution is $\mathcal D_P:=\mathcal D(\cdot\mid\mathcal G(a)=P)$ and its arrival rate $\alpha_{\iota_P}:=\Pr_{a\sim\mathcal D}(\mathcal G(a)=P)$. There are $G$ base groups plus the universal group, but up to $2^G$ intersections; $H$ is the number with positive arrival rate and $H_k$ the number with $k\in P$. Following sections evaluate two specific instantiations of this structure: a two-group setting ($\mathcal A=\{M,m\}$, with non-universal groups $g_M,g_m$)% partitioning the population)
 and the pattern-only setting of our experiments ($\mathcal A=\{0,1\}^G$, $\phi_k(a)=a_k$).

\noindent \paragraph{Items and Valuations} 
\label{sec:model-items}
There are $n$ items, indexed $j \in [n] := \{1,\dots,n\}$. Conditional on $a_i$, user $i$'s valuation $v_{ij}$ for item $j$ is drawn independently across users and items from a distribution on $[0,V_{\max}]$, where $V_{\max} > 0$. We define the expected welfare of item $j$ conditional on the full attribute vector as $v_{j,a} := \E[v_{ij} \mid a_i = a]$. From this, we define the expected welfare of item $j$ over any conditioning distribution $\mathcal D'$ as $v_j(\mathcal D') := \E_{a \sim \mathcal D'}[v_{j,a}]$. This defines the population mean $v_j^*:=v_j(\mathcal D)$, the group mean $v^*_{j,k}:=v_j(\mathcal D_k)$, and the intersection mean $v^*_{j,\iota_P}:=v_j(\mathcal D_P)$.

% The catalog contains $n$ items, indexed $j \in [n] := \{1,\dots,n\}$. We assume there is a finite $V_{\max}>0$ bounding the maximum possible valuation (or rating) user $i$'s valuation $v_{ij}$ for item $j$ can take; i.e., $0\le v_{ij}\le V_{\max}$.
%  Conditional on $a_i$, user $i$'s valuation $v_{ij}$ for item $j$ is drawn independently across users and items from a distribution on $[0,V_{\max}]$. We define the expected welfare of item $j$ conditional on the full attribute vector as $v_{j,a} := \E[v_{ij} \mid a_i = a]$. From this, we define the expected welfare of item $j$ over any conditioning distribution $\mathcal D'$ as $v_j(\mathcal D') := \E_{a \sim \mathcal D'}[v_{j,a}]$. This establishes the optimality concepts used throughout:
% \begin{itemize}
% \item \textbf{Population:} $v_j^*:=v_j(\mathcal D)$.
% \item \textbf{Group:} $v^*_{j,k}:=v_j(\mathcal D_k)$.
% \item \textbf{Intersection:} $v^*_{j,\iota_P}:=v_j(\mathcal D_P)$.
% \item \textbf{Attribute-conditional:} $v_{j,a}$ is the conditional mean of item $j$ given $a$ (not itself an optimum).
% \end{itemize}
% By total expectation, the pooled marginal mean for a base group $g_k$ is linked to its constituent intersections via $v^*_{j,k}=\sum_{P\ni k}(\alpha_{\iota_P}/\alpha_k)\,v^*_{j,\iota_P}$.

\noindent \paragraph{Arrivals and Slates}
\label{sec:model-dynamics}
% The arrival rate of a group $g_k$ is the probability that a randomly arriving user belongs to it, denoted $\alpha_k := \Pr_{a \sim \mathcal{D}}(\phi_k(a) = 1)$. Consequently, the arrival rate of an intersection is $\alpha_{\iota_P} := \Pr_{a\sim\mathcal{D}}(\mathcal{G}(a)=P)$.
In each round $t$, the platform recommends a slate $S_t\subseteq[n]$ of exactly $K$ distinct items, where $1\le K\le n$.

\subsection{Feedback Model} 
\label{sec:model-feedback}
Upon serving the slate $S_t$, the platform observes the user's score $v_{i_tj}$ only for the recommended items $j\in S_t$; write $Y_{j,t}:=v_{i_tj}\mathbf 1[j\in S_t]$. Scores corresponding to unselected items remain unobserved. The scores of selected items may update any model the policy keeps, but this remains partial feedback. The history is the filtration $\mathcal F_t:=\sigma\bigl(U_s,S_s,\{Y_{j,s}\}_{j\in S_s},\mathcal G^{\rm obs}_s : s<t\bigr)$, which includes algorithmic randomness $U_s$ and, for group-aware policies, past observed memberships $\mathcal G^{\rm obs}_s$. The platform's realized welfare at round $t$ is the slate total, $W_t:=\sum_{j\in S_t}Y_{j,t}$.

\subsection{Recommendation Policies}
\label{sec:model-recommendation}
We consider both group-blind and group-aware policies. 
A (possibly randomized) group-blind policy draws $S_t$ from a distribution $\nu_t(\cdot\mid\mathcal F_t)$ that does not depend on the current user's information. A group-aware policy may additionally observe which groups in $\mathcal G$ are active for the arriving user $i$---represented by the active set $\mathcal{G}_i$---and condition its recommended items on it. 

\paragraph{Baselines and UCB scopes}
We use \textbf{group-blind top-$K$ UCB} as the baseline throughout the paper. Serve the $K$ items with the largest UCB indices \cite{auer2002finite}; ties are broken by a fixed priority order, the same for every method: base groups before the universal group, items by index, and (group, item) pairs lexicographically by group and then item. The full algorithm is given below: 

\begin{algorithm}[ht] 
\caption{Group-Blind top-$K$ UCB} 
\label{alg:group_blind_ucb} 
\textbf{Initialize:} For all items $j \in [n]$, set $N_{j,0} = 0$ and $\hat v_{j,0} = 0$. \\ 
\textbf{For} each round $t=0, 1, \dots, T-1$: 
\begin{enumerate}
\item \textbf{Index Computation:} For every item $j$, if $N_{j,t} = 0$, set $I_{j,t} = +\infty$. Otherwise, compute the index using the global clock $t+1$: 
\[ I_{j,t} = \hat v_{j,t} + V_{\max}\sqrt{\frac{2\ln\max\{2,t+1\}}{N_{j,t}}} \]
\item \textbf{Serve:} Select the slate $S_t \subseteq [n]$ of the $K$ distinct items with the highest indices $I_{j,t}$ (ties broken by smaller item index) and recommend it.
\item \textbf{Observe Feedback:} Observe the realized welfare $Y_{j,t} \in [0, V_{\max}]$ only for $j \in S_t$.
\item \textbf{Update Statistics:} For each $j \in S_t$: $N_{j,t+1} = N_{j,t} + 1$ and $\hat v_{j,t+1} = \hat v_{j,t} + \frac{Y_{j,t} - \hat v_{j,t}}{N_{j,t+1}}$. For $j \notin S_t$, carry values forward.
\end{enumerate} 
\end{algorithm}

% Each instance keeps, for every item $j$, a count $N$ of observed scores, their mean $\hat v$, and a clock $\tau$ equal to the number of rounds, up to and including the current one, in which the instance is updated. The index $I$ of an item is $+\infty$ if $N=0$, and otherwise
% \begin{equation}\label{eq:ucb} \hat v+V_{\max}\sqrt{\frac{2\ln\max\{2,\tau\}}{N}}. \end{equation}

% To isolate the effect of group awareness, we instantiate top-$K$ UCB at different demographic scopes. 
% Each scope tracks its own counts $N$, means $\hat v$, clock $\tau$, and index $I$, denoted by their subscripts.
We introduce three group-aware algorithms in Section~\ref{sec:groupaware} and study their performance in Section~\ref{sec:experiments}:

\begin{itemize}
\item \textbf{Intersection-Fitting UCB} (Section~\ref{sec:groupaware}, Algorithm~\ref{alg:isfit}): Maintains an independent UCB instance for every realized intersection, with no data sharing across intersections.
\item \textbf{Inner--Outer UCB} (Section~\ref{sec:overlapping}, Algorithm~\ref{alg:hier}): Tracks statistics at the base-group level and routes each user to the active group whose proposed top-$K$ slate has the highest average item index.
\item \textbf{Group--Item UCB} (Section~\ref{sec:overlapping}, Algorithm~\ref{alg:flat}): Treats active (group, item) pair as independent arms, allowing a slate to combine items favored by different overlapping groups.
\end{itemize}

\subsection{Benchmarks and Metrics }
\label{sec:model-benchmarks-metrics}
\paragraph{Metrics.} 
Real platforms recommend sets of items, so we formulate our metric directly for a general set size $K$. The highest expected welfare a fully group-blind policy can achieve is the population-optimal $K$-slate $B^K_{\rm pop} := \max_{|S|=K}\sum_{j\in S}v_j^*$. Similarly, for a specific base group $g_k$, the group-optimal slate welfare is $v_k^{*,K}:=\max_{|S'|=K}\sum_{j\in S'}v_{j,k}^*$.

For any attribute vector $a$, the analogous attribute-conditional optimal slate welfare is $v^{*,K}_a:=\max_{|S'|=K}\sum_{j\in S'}v_{j,a}$. Since a fixed recommendation for a group cannot exceed the welfare of perfectly personalizing to each member's exact attributes, the group-optimal slate welfare is bounded by the expected attribute-conditional optimum: $v_k^{*,K} \le \E_{a\sim\mathcal D_k}[v^{*,K}_a]$. For an exact intersection $\iota_P$, the intersection-optimal slate welfare is similarly defined as:
\[ v_{\iota_P}^{\mathrm{shared},K} := \max_{|S'|=K}\sum_{j\in S'}v_{j,\iota_P}^*. \]
To simplify notation in proofs and summations, let $\iota$ index an exact intersection. Let $p_\iota:=\alpha_{\iota}$, $\mu_{\iota,j}:=v^*_{j,\iota}$, and $b_\iota:=v^{\mathrm{shared},K}_{\iota}$. A policy that observes only the demographic membership $\mathcal G(a)$ has an expected per-round welfare of at most the intersection-level oracle $B_{\mathrm{int}}:=\sum_\iota p_\iota b_\iota$. In contrast, the fully personalized ceiling $B^{\mathcal G}_R:=\E_{a\sim\mathcal D}[v^{*,K}_a]$ defines total regret. Since the maximum of an average is at most the average of maxima, $B^{\mathcal G}_R\ge B_{\mathrm{int}}\ge B^K_{\rm pop}$.

%Group and intersection regret use their respective best fixed $K$-slate benchmarks. The ceiling uses conditional expected item values, not the user's realized scores.

We evaluate policies against three regret metrics.

\begin{definition}[Total Regret]
\label{def:total_regret}
The cumulative total regret measures the aggregate platform shortfall against the global oracle:

\[
\Rtot(T) := \sum_{t=0}^{T-1} \bigl(B^{\mathcal G}_R - \E[W_t]\bigr).
\]
\end{definition}
Total regret measures aggregate platform efficiency, but obscures performance on distinct demographic groups.

\begin{definition}[Group Regret]
\label{def:group_regret}
For base group $g_k$, the cumulative group regret restricts evaluation to rounds in which $g_k$ is active, measured against the group's own slate benchmark:
\[
    R_k(T) := \sum_{t=0}^{T-1} \E\bigl[\mathbf 1\{g_k \in \mathcal G_{i_t}\}\cdot(v_k^{*,K} - W_t)\bigr].
\]
A policy achieves \emph{no group regret} if $\limsup_{T\to\infty}R_k(T)/(\alpha_kT)\le0$ for every
fixed group with $\alpha_k>0$; equivalently $[R_k(T)]_+=o(\alpha_kT)$. %Group regret can be negative when personalization beats the group's best fixed slate.
\end{definition}

\begin{definition}[Intersection Regret]
\label{def:intersection_regret}
For an intersection $\iota_P$, restrict evaluation to rounds where the arriving user's active set is exactly $P$. We judge algorithms against the shared-slate benchmark:
\[
    R_{\iota_P}^{\mathrm{shared}}(T):=\sum_{t=0}^{T-1}\E\bigl[\mathbf 1\{\mathcal G_{i_t}=P\}\cdot(v_{\iota_P}^{\mathrm{shared},K}-W_t)\bigr].
\]
A policy achieves no intersection regret if $R^{\mathrm{shared}}_{\iota_P}(T) = o(\alpha_{\iota_P}T)$ for every fixed intersection with positive arrival rate.
\end{definition}

Note that group regret can be negative when personalization at the intersection level beats the overall group's best \emph{fixed} slate.

\section{The Structural Cost of Group-Blindness}
\label{sec:group_blind}
% Group-blind algorithms pay a structural penalty that does not shrink with additional exploration. To illustrate this, we specialize the general model to a simplified instance with $n=2$ items, $K=1$ recommendations per round, and two base groups, $g_M$ and $g_m$, with arrival shares $\alpha := \alpha_M > 1/2$ and $\alpha_m = 1-\alpha$. Valuations $v_{ij} \mid g$ are drawn independently, yielding a pooled population expected welfare for item $j$ of $v_j^* := \alpha v_{j,M}^* + (1-\alpha)v_{j,m}^*$.

% We fix an instance satisfying the divergent-peaks condition $\mathcal{E} := \{v_{1,M}^* > v_{2,M}^*\} \cap \{v_{2,m}^* > v_{1,m}^*\}$, with strictly positive group-level gaps $d_M := v_{1,M}^* - v_{2,M}^*$ and $d_m := v_{2,m}^* - v_{1,m}^*$. Under $\mathcal{E}$, the fully personalized ceiling is $B^{\mathcal G}_R = \alpha v^*_{1,M} + (1-\alpha) v^*_{2,m}$. Because a group-blind policy treats the population as a monolith, it evaluates items based on their pooled expected welfare $v_j^*$. To illustrate the majority-favoring case, assume $\alpha d_M > (1-\alpha)d_m$, which ensures $v_1^* > v_2^*$. Under this condition, a population-level no-regret guarantee ensures the algorithm makes a vanishing fraction of suboptimal selections, converging toward item $1$. This restricts the minority group to their less-preferred item, creating an efficiency gap that demographic-aware routing could recover.

% This structural penalty is not limited to two groups or single items. Returning to the general model with arbitrarily many overlapping groups and slate size $K$, 

We start with a mild assumption on user preferences:
\begin{assumption}[Non-degeneracy]\label{ass:nondegenerate}
Let $S^*$ be the unique population-optimal $K$-slate. Some group $g_k$ satisfies $v^{*,K}_k >\sum_{j\in S^*}v^*_{j,k}$.
\end{assumption}

% This assumption ensures that the population-optimal slate is strictly suboptimal for at least one demographic group. Without this condition, a single fixed slate would simultaneously maximize welfare for every group, and group-blindness would incur no structural penalty.

This assumption ensures the population-optimal slate is strictly suboptimal for at least one demographic group. Otherwise, a single fixed slate would maximize welfare for every group, and group-blindness would incur no structural penalty.

We show that any group-blind policy %---including sharper exploration rules like UCB---
inherits a strict $\Theta(T)$ regret penalty. This gap arises whenever the population-optimal slate is strictly suboptimal for a specific demographic segment; optimizing for the ``average'' user yields a linear welfare penalty.

\begin{proposition}[Universal structural gap]
\label{prop:universal}
For every group-blind policy, $\E[W_t\mid\mathcal F_t]\le B^K_{\rm pop}$, hence $\Rtot(T)\ge T\,(B^{\mathcal G}_R-B^K_{\rm pop})$. Under Assumption~\ref{ass:nondegenerate}, the group-aware oracle exceeds the population optimum ($B_{\mathrm{int}} > B^K_{\rm pop}$), yielding a strict $\Theta(T)$ structural penalty.
\end{proposition}

\begin{proof}
Because the policy is group-blind, the served slate $S_t$ is independent of the arriving user's valuations conditional on the history $\mathcal F_t$. Therefore, $\E[W_t\mid\mathcal F_t]=\E\bigl[\sum_{j\in S_t}v_j^*\mid\mathcal F_t\bigr]\le B^K_{\rm pop}$. Summing over all rounds yields the lower bound on total regret.

To prove $B_{\mathrm{int}} - B^K_{\rm pop} > 0$, let $g_k$ be a group satisfying Assumption~\ref{ass:nondegenerate}. Consider a group-aware policy that serves the group's optimal slate when $g_k$ is active, and the population-optimal slate $S^*$ otherwise. This policy achieves an expected welfare exceeding $B^K_{\rm pop}$ by exactly $\alpha_k\bigl(v_{k}^{*,K} - \sum_{j \in S^*} v_{j,k}^*\bigr) > 0$. As $B_{\mathrm{int}}$ is the maximum expected welfare achievable by any group-aware policy, we get $B_{\mathrm{int}} > B^K_{\rm pop}$.
\end{proof}

\begin{proposition}[Constant gap]
\label{prop:constantgap}
Let $r=\min\{K,n-K\}$. There is an instance on which every group-blind policy has $\Rtot(T)\ge\tfrac{rV_{\max}}{3}T$.
\end{proposition}

\begin{proof}
Take two groups of mass $2/3$ and $1/3$. $K-r$ common items are worth $V_{\max}$ to everyone; $r$ items are worth $V_{\max}$ only to the majority; $r$ items are worth $V_{\max}$ only to the minority; any extra items are worth $V_{\max}/10$. The group-aware optimum is $KV_{\max}$; the blind optimum is $(K-r)V_{\max}+\tfrac{2r}{3}V_{\max}$. The gap is $r V_{max}/3$.% when $n\ge2K$. %At $K=n$ regret is zero.
\end{proof}

\section{Group-Aware Approaches}
\label{sec:groupaware}
To resolve the structural penalty of group-blind algorithms, a platform must condition recommendations on user demographics. With disjoint groups, one can run one no-regret learner (e.g., UCB) per group. However, real-world users belong to multiple intersecting groups simultaneously. Maintaining an independent model for every demographic intersection requires up to $O(n 2^G)$ state and can slow finite-horizon learning on rare combinations. To address this, we study three approaches. All serve exactly $K$ distinct items, observe only the scores of served items, and compute upper confidence bounds at the item level as detailed in their respective algorithms.

% To address this, we study three approaches. Intersection fitting keeps separate statistics for every exact membership intersection. Two group-based methods keep $O((G+1)n)$ item statistics: Inner--Outer UCB, which routes each user to one awake group, and Group--Item UCB, which treats every (group, item) pair as an arm. All serve exactly $K$ distinct items, observe only the scores of served items, and use the top-$K$ UCB index~\eqref{eq:ucb} of Section~\ref{sec:model-recommendation} at the item level.

\subsection{Exact Intersection-Fitting}
\label{sec:intersection_fitting}
The most direct approach is to maintain a separate learner for every exact demographic intersection. Intersection-fitting shares no data across intersections and uses $O(nH)$ state, where $H \le 2^G$ is the number of intersections with a positive arrival rate.

\begin{algorithm}[ht] 
\caption{Intersection-Fitting UCB} 
\label{alg:isfit} 
\textbf{Initialize:} For each exact intersection $\iota_P$ on its first appearance, set $\hat v_{j,\iota_P}=0$ and $N_{j,\iota_P}=0$ for all $j\in[n]$.\\ 
\textbf{For} each round $t=0, 1, \dots, T-1$: 
\begin{enumerate}
\item \textbf{Observe Context:} The active set is $P=\mathcal G_{i_t}$; the user belongs to intersection $\iota_P$. Let $\tau_{\iota_P}(t)$ be the number of arrivals in $\iota_P$ up to and including the current round $t$.
\item \textbf{Index Computation:} For every item $j$, if $N_{j,\iota_P} = 0$, set $I_{j,\iota_P,t} = +\infty$. Otherwise, compute the index using the local intersection clock $\tau_{\iota_P}(t)$:
\[ I_{j,\iota_P,t} = \hat v_{j,\iota_P} + V_{\max}\sqrt{\frac{2\ln\max\{2,\tau_{\iota_P}(t)\}}{N_{j,\iota_P}}} \]
\item \textbf{Serve:} Recommend the $K$ items with the highest indices $I_{j,\iota_P,t}$ (ties broken by smaller item index).
\item \textbf{Observe Feedback:} Observe $Y_{j,t}$ for each served item $j\! \in \!S_t$.
\item \textbf{Update Statistics:} For each $j\in S_t$, update its statistics: $N_{j,\iota_P} \leftarrow N_{j,\iota_P} + 1$ and $\hat v_{j,\iota_P} \leftarrow \hat v_{j,\iota_P} + \frac{Y_{j,t} - \hat v_{j,\iota_P}}{N_{j,\iota_P}}$. All other intersections are unchanged.
\end{enumerate} 
\end{algorithm}

\begin{proposition}[Intersection fitting]
\label{prop:isfit}
For every $1\le K\le n$ and $T\ge1$, Algorithm~\ref{alg:isfit} satisfies, for every intersection $\iota_P$ with $\alpha_{\iota_P}>0$,
\[ R^{\rm shared}_{\iota_P}(T)\le V_{\max}\Bigl[4\sqrt{2\,nK\,\alpha_{\iota_P}T\ln(2T)}+2nK\Bigr], \]
and, for every group $g_k$ with $\alpha_k>0$,
\[ R_k(T)\le V_{\max}\Bigl[4\sqrt{2\,nKH_k\,\alpha_kT\ln(2T)}+2nKH_k\Bigr]. \]
Hence it achieves no intersection regret and no group regret. %, and at $K=n$ regret is exactly $0$.
\end{proposition}

Proof in Appendix~\ref{app:isfit}. While intersection-fitting guarantees no-intersection regret (hence no-group regret), its theoretical upper bound scales with $O(\sqrt{H_k})$. Without a matching lower bound, this does not prove the algorithm is strictly slow in all environments; rather, it indicates that in the worst case, an intersection-fitting algorithm may be too slow, learning at a rate that depends exponentially on the number of groups $G$, fragmenting its exploration budget and failing to learn from higher group-level structures.

\subsection{Group-Based Pooling Heuristics}
\label{sec:overlapping}
To address this worst-case risk, we adapt two group-based heuristics that learn at the base-group level and pool data to serve users within each intersection. Rather than track every intersection, they maintain $G+1$ item tables---one per base group plus the always-active baseline $g_0$---using $O((G+1)n)$ item statistics.

 %Their main reduction updates experts using an observed outcome that evaluates each expert's loss; a single served item's score does not evaluate the other groups' proposals, so their guarantees do not transfer to our selected-item feedback.

\begin{algorithm}[!h] 
\caption{Inner--Outer UCB} 
\label{alg:hier} 
\textbf{Initialize:} For every group $g_k$ ($k\in\{0,\dots,G\}$): $\hat v_{j,k}=0$ and $N_{j,k}=0$ for all $j$. \\ 
\textbf{For} each round $t=0, 1, \dots, T-1$: 
\begin{enumerate}
\item \textbf{Observe Context:} The awake groups are those in $\mathcal G_{i_t}$.
\item \textbf{Index Computation:} For every awake group $g_k$ and item $j$, if $N_{j,k} = 0$, set $I_{j,k,t} = +\infty$. Otherwise, compute the UCB-style index using the global clock $t+1$:
\[ I_{j,k,t} = \hat v_{j,k} + V_{\max}\sqrt{\frac{2\ln\max\{2,t+1\}}{N_{j,k}}} \]
Let $S_{k,t} \subseteq [n]$ be the $K$-slate with the highest UCB-style indices for group $g_k$ (ties broken by smaller item index).
\item \textbf{Group Selection:} Compute the routing score $Q_{k,t} := \frac{1}{K}\sum_{j\in S_{k,t}} I_{j,k,t}$ for each awake group. Follow the awake group $g_{k^*}$ with the largest routing score $Q_{k,t}$, breaking ties by smaller group index.
\item \textbf{Serve:} Recommend the slate $S_t = S_{k^*,t}$.
\item \textbf{Observe Feedback:} Observe $Y_{j,t}$ for each served item $j\!\in \!S_t$.
\item \textbf{Update (followed group only):} For each $j\in S_t$, update the followed group's statistics: $N_{j,k^*} \! \leftarrow \!N_{j,k^*} \!+\! 1$ and $\hat v_{j,k^*} \leftarrow \hat v_{j,k^*} + \frac{Y_{j,t} - \hat v_{j,k^*}}{N_{j,k^*}}$. All other groups remain unchanged.
\end{enumerate} 
\end{algorithm}

\paragraph{Inner--Outer UCB} Our first heuristic, given in Algorithm~\ref{alg:hier}, builds on sleeping-bandit ideas \cite{kleinberg2010sleeping, blum2019advancing} and extends a technique from \cite{acharya2023oracle} to sequential recommendation. It maintains a learner for each group, augmented with an optimistic routing heuristic that determines which group's recommendation rule to follow when multiple are active. 
The inner step computes UCB-style indices to maintain an optimistic estimate for each item within a group, selecting the top $K$ items. For the outer selection, rather than implementing a separate AUER meta-learner \cite{kleinberg2010sleeping}, the algorithm uses an optimistic routing heuristic that picks the group whose proposed top-$K$ items have the highest average UCB-style index. Because a group's pooled empirical mean is a selection-weighted average of its constituent intersections' means, these group-level updates are biased, and we do not inherit the general no-regret guarantee of \citet{acharya2023oracle} in the worst case\footnote{A construction with linear regret is given in Appendix~\ref{app:counter}}. However, our experiments in Section~\ref{sec:experiments} show that our pooling framework still performs well when intersections inherit the valuation structure of their higher-level groups.

\begin{algorithm}[ht] 
\caption{Group--Item UCB} 
\label{alg:flat} 
\textbf{Initialize:} For every (group, item) pair: $\hat v_{j,k}=0$ and $N_{j,k}=0$.\\ 
\textbf{For} each round $t=0, 1, \dots, T-1$: 
\begin{enumerate}
\item \textbf{Observe Context:} The awake groups are those in $\mathcal G_{i_t}$.
\item \textbf{Index Computation:} For every awake group $g_k$ and item $j$, if $N_{j,k} = 0$, set $I_{j,k,t} = +\infty$. Otherwise, compute the UCB-style index using the global clock $t+1$:
\[ I_{j,k,t} = \hat v_{j,k} + V_{\max}\sqrt{\frac{2\ln\max\{2,t+1\}}{N_{j,k}}} \]
An item's score is its maximum UCB-style index over the awake groups, attained by its \emph{winning group}.
\item \textbf{Serve:} Recommend the $K$ distinct items with the highest scores (breaking ties by smaller group index, then smaller item index).
\item \textbf{Observe Feedback:} Observe $Y_{j,t}$ for each served item $j\!\in \!S_t$.
\item \textbf{Update (winning pair only):} For each $j\in S_t$, update only the statistics of its specific winning group $g_k$: $N_{j,k} \leftarrow N_{j,k} + 1$ and $\hat v_{j,k} \leftarrow \hat v_{j,k} + \frac{Y_{j,t} - \hat v_{j,k}}{N_{j,k}}$.
\end{enumerate} 
\end{algorithm}

\paragraph{Group--item UCB} 
Our second heuristic treats awake (group, item) pair as independent arms, allowing a single slate to combine items from different groups. The algorithm computes a UCB-style index for each item under every awake group. An item's overall score is its maximum 
% UCB-style 
index across these awake groups, attained by the winning group. The algorithm selects the $K$ distinct items with the highest overall scores and updates the winning (group, item) pair for each served item. As with Inner--Outer UCB, this algorithm lacks formal regret guarantees because a group's pooled empirical mean is a selection-weighted average of its constituent intersections' means (weighted by intersection-specific selection rates), rather than the true group  mean. Its indices therefore act as heuristic scores. The full algorithm is described in Algorithm~\ref{alg:flat}.

%routes each user to the awake group whose proposed slate has the highest average index and serves that slate; Group--item UCB lets every item be proposed by whichever awake group is most optimistic about it, so one slate can combine items favored by different groups. Note that a group's pooled mean is a selection-weighted average of its intersections' means (weighted by intersection selection rates), not the group's population mean, so its index acts as a heuristic score.

%We emphasize that these two group-based heuristics lack any theoretical regret guarantees. 
%\paragraph{Relationship between Inner--Outer UCB and Group--item UCB} At $K=1$, Inner--Outer UCB and Group--item UCB coincide: under the common fixed priority order and the global clock, both select the awake (group, item) pair with the highest index and update only that pair, so their trajectories are equal. %For this $K=1$ case a constructed instance with overlapping groups gives linear group regret (Appendix~\ref{app:counter}); we therefore rely on experiments to show when pooling succeeds and fails at larger $K$.

\subsection{Limits of Intersectional Learning}
Any recommendation policy serving this heterogeneous population faces an information-theoretic limit on how well it can serve any specific base group, dictated by how often that group arrives.

\begin{theorem}[Worst-case group lower bound]
\label{thm:lower_bound}
For any $1\le K<n$, any group frequency $\alpha_k\in(0,1]$ with $\alpha_kT\ge1$, and any algorithm using only selected-item feedback, some instance in the model has $R_k(T)\ge\frac{V_{\max}}{16e}\sqrt{\alpha_kT}$.
\end{theorem}
Proof in Appendix~\ref{app:lower}. 
%This is a worst-case existence bound, not a matching minimax characterization in $n,K,G$ or an upper guarantee for any method.
The hard environments differ only in group $g_k$'s valuations, so information accrues only on the roughly $\alpha_kT$ rounds where $g_k$ is present. We emphasize this is a worst-case limit; it does not preclude algorithms from achieving lower regret on specific empirical instances.

\section{Experiments}
\label{sec:experiments}
We evaluate three dynamics of intersectional learning: welfare gains of demographic conditioning over group-blind baselines, conditions under which pooling improves finite-horizon performance over exact intersection-fitting, and how relative advantages shift across slate sizes and evaluation levels.

%\subsection{Datasets}
%\label{subsec:datasets}
\paragraph{Datasets} We evaluate two rating simulations to test our heuristics under different preference structures. The Goodreads environment evaluates performance when demographic intersections inherit structured preferences from constituent base groups (e.g., via Lancaster characteristics). For BoardGameGeek (BGG), eight games are selected per theme to construct the catalog, and reward distributions are fitted at the intersection level. BGG imposes no rule composing intersection preferences from constituent group tastes, allowing exact intersections to exhibit heterogeneous rankings completely different from their broader population. Both constructions fit the same conditional-reward model. In both datasets, $\mathcal{D}$ is the empirical distribution of membership patterns, and attribute vectors are the membership pattern ($a \in \{0,1\}^G$ with $\phi_k(a)=a_k$), so the experiments represent the special case where demographic identities act as the sole user features; thus $B^{\mathcal G}_R=B_{\mathrm{int}}$.

% \paragraph{Datasets} We evaluate two rating simulations to test our heuristics under different preference structures. The Goodreads environment evaluates performance when demographic intersections inherit structured preferences from constituent base groups (e.g., via Lancaster characteristics). The BoardGameGeek (BGG) environment evaluates performance when exact intersections within a group can have completely different, heterogeneous properties from the broader population.
% Eight games are selected per theme to make the catalog; the reward distributions are then fitted at the intersection level. BGG imposes no corresponding rule that composes intersection preferences from constituent group tastes, allowing intersection rankings to differ from those of their containing groups. 
% Both constructions fit the same conditional-reward model. In both datasets, $\mathcal{D}$ is the empirical distribution of membership patterns, and attribute vectors are the membership pattern ($a \in \{0,1\}^G$ with $\phi_k(a)=a_k$), so the experiments represent the special case where demographic identities act as the sole user features; thus $B^{\mathcal G}_R=B_{\mathrm{int}}$ there. We consider the two following datasets:

\begin{itemize}
\item \emph{Goodreads:}
We construct the base-group topology $\mathcal{G}$ directly from the observed co-membership network of $G=8$ Goodreads book clubs \cite{wan2018item, wan2019fine}. A user's active set $\mathcal{G}_i$ is their actual club membership profile, ensuring every evaluated intersection $\iota$ is a naturally occurring demographic pattern. Because Goodreads does not expose per-book ratings tied to specific club intersections at our required scale, we construct item valuations using a Lancaster latent characteristics model \cite{lancaster1966new}. Books are embedded as feature vectors representing their composition across genres, and user tastes are modeled as centroids of their active clubs. To ensure intra-genre variation, items follow fixed composition tiers (e.g., pure flagships vs. cross-genre hybrids). 
The evaluated $15$-item environment is formed by taking a fixed $\lceil 0.2 \times 72 \rceil = 15$ item uniformly sampled subset of the full catalog (catalog seed 2026), retaining original item order and recomputing benchmarks on the subset.
% We run two catalog sizes: $n=15$ (main text) and $n=72$ (Appendix~\ref{app:extended_graphs}). A smaller $15$-item environment is formed by taking a fixed $\lceil 0.2 \times 72 \rceil = 15$ item uniformly sampled subset (catalog seed 2026), retaining original item order and recomputing benchmarks on the subset.
\item \emph{BoardGameGeek:}
We use a public BoardGameGeek (BGG) database \cite{kaggle_bgg} containing hundreds of thousands of user ratings across $G=9$ overarching themes. To construct an environment designed to reward personalization without altering the underlying welfare metric, we filter the catalog to explicitly remove universally liked items. 
\end{itemize}

The remaining polarizing games are uniquely assigned to the theme where they over-index most strongly. 
The evaluated $15$-item environment is formed by taking a fixed $\lceil 0.2 \times 72 \rceil = 15$ item uniformly sampled subset of the catalog (catalog seed 2026), retaining item order and recomputing benchmarks on the subset.
% We run two catalog sizes: $n=15$ (main text) and $n=72$ (Appendix~\ref{app:extended_graphs}). A smaller $15$-item environment is formed by taking a fixed $\lceil 0.2 \times 72 \rceil = 15$ item uniformly sampled subset (catalog seed 2026), retaining original item order and recomputing benchmarks on the subset. 
Reward centers are the empirical mean ratings of users in each exact intersection (global fallback if empty); realized rewards are Gaussian draws around them, clipped to $[1,10]$ with a standard deviation floor of $0.5$; expected welfare is the post-clipping mean. Full catalog filtering and calibration details are provided in Appendix~\ref{app:simulation}.

%\subsection{Methodology}
%\label{subsec:methodology}
% Every experiment compares the same four policies, all defined for general $K$:
% \begin{enumerate}
% \item \textbf{Group-Blind UCB} (Algorithm~\ref{alg:group_blind_ucb}): a single pooled UCB instance over the full catalog, without demographic context.
% \item \textbf{Intersection-Fitting UCB} (Algorithm~\ref{alg:isfit}): an independent UCB instance for every realized intersection, with no data sharing across intersections.
% \item \textbf{Inner--Outer UCB} (Algorithm~\ref{alg:hier}): each user is routed to the awake group whose proposed top-$K$ slate has the highest average item index, and that slate is served; only the followed group is updated.
% \item \textbf{Group--Item UCB} (Algorithm~\ref{alg:flat}): joint UCB indices over all awake (group, item) pairs; only each served item's winning pair is updated.
% \end{enumerate}

\paragraph{Methodology} Each experiment compares the four policies defined earlier: Group-Blind UCB, Intersection-Fitting UCB, Inner-Outer UCB, and Group-Item UCB. We report cumulative group regret $R_k(T)$ and intersection regret $R^{\rm shared}_{\iota_P}(T)$, measured against the Section~\ref{sec:model-benchmarks-metrics} benchmarks evaluated at post-clipping means. Each curve is a single seed, estimating regret from realized rewards (benchmark minus served slate welfare), so individual curves can dip below zero from noise. Group regret can also be genuinely negative (Definition~\ref{def:group_regret}): by $R_k=\sum_{\iota\ni k}R_\iota-T\Gamma_k$, this happens when intersection-specific slates beat the group's best fixed slate by more than the intersections' own regret, without making intersection regret negligible. Overlapping group regrets cannot be summed to obtain platform regret, since users are counted once per group; sum intersection regrets instead. All policies see the same arrival sequence. At $K=1$ the two pooled heuristics coincide, yielding identical curves. For exact reproducibility, all plotted trajectories use arrival seed 3000, reward seed 1003000, and catalog seed 2026.

\subsection{Goodreads: Performance with Inherited Group Structure}
\label{subsec:goodreads_results}

The Goodreads environment explicitly constructs intersectional preferences as the average of their constituent base-club tastes. Because the intersections inherit this clear group structure, sharing data across overlapping groups is generally highly effective. 
As shown across all slate sizes (with $K=1$ and $5$ in Figures~\ref{fig:goodreads_k1} and \ref{fig:goodreads_k5}, and intermediate $K=3$ results showing consistent trends deferred to Appendix~\ref{app:extended_graphs}), Group-Blind UCB incurs strictly linear regret, focusing on the population average and starving minority clubs.
Intersection-Fitting guarantees sublinear regret, but its theoretical worst-case parameters scale poorly; its fragmented exploration budget risks slow learning at finite horizons, which we observe empirically, especially on rare intersections.

% ==========================================
% GOODREADS n=15 (K=1)
% ==========================================
\begin{figure*}[h]
    \centering
    \begin{subfigure}{\textwidth}
    \centering
        \includegraphics[width=0.9\textwidth]{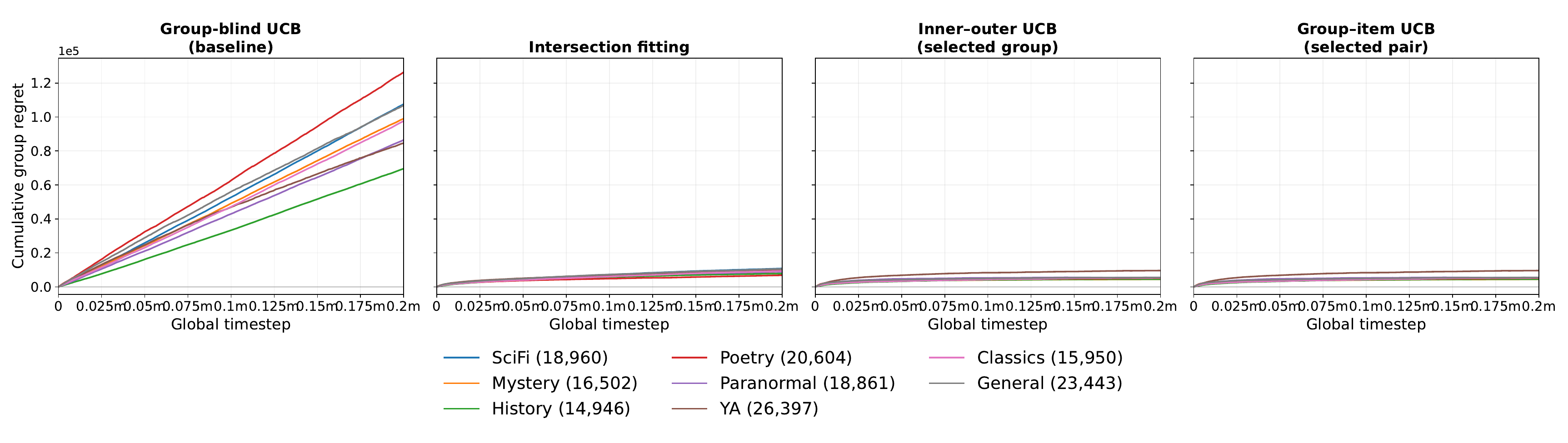}
        \caption{Base Clubs}
    \end{subfigure}
    % \vspace{0.2cm}
    
    \begin{subfigure}{\textwidth}
    \centering
        \includegraphics[width=0.9\textwidth]{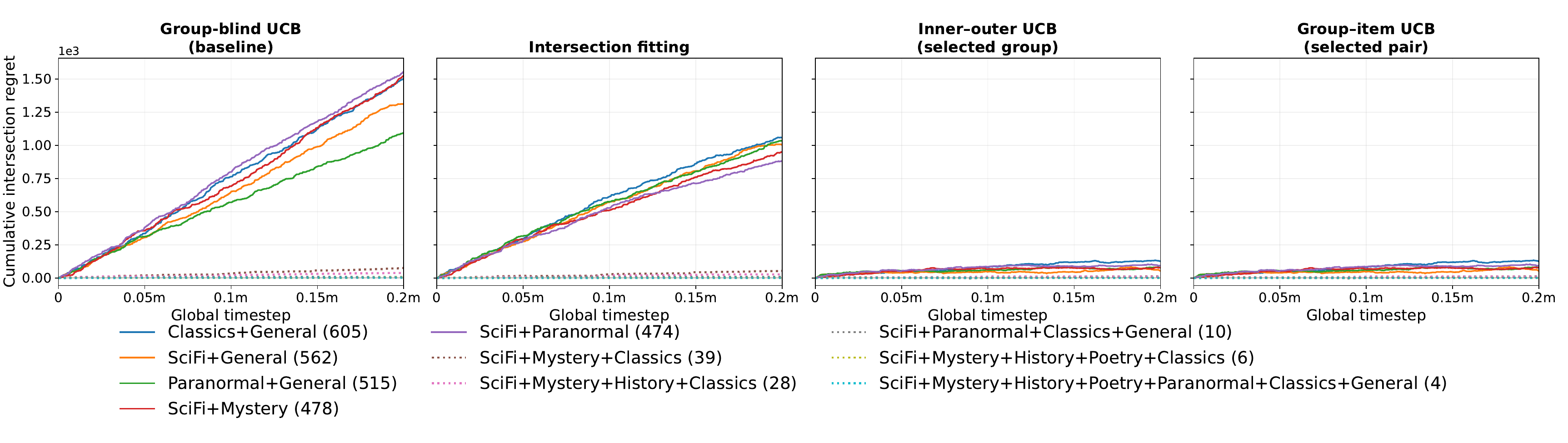}
        \caption{Common (Solid Lines) and Rare (Dotted Lines) Intersections}
    \end{subfigure}
    % \vspace{0.2cm}

    % \begin{subfigure}{\textwidth}
    %     \includegraphics[width=\textwidth]{outputs/goodreads/n=15/Goodreads_15items_200000rounds_K1_rare_regret.pdf}
    %     \caption{Rare Intersections}
    % \end{subfigure}
    
   \caption{Goodreads primary catalog ($n=15$) at slate size \textbf{$K=1$}. At the group level, Intersection-Fitting and both pooling heuristics (Group--Item UCB and Inner--Outer UCB) achieve comparable low cumulative regret. At the intersection level, common demographics (solid lines) and rare demographics (dotted lines) are plotted together. Across both traffic tiers, the pooling heuristics achieve strictly lower cumulative regret than exact Intersection-Fitting.}
    \label{fig:goodreads_k1}
\end{figure*}

% ==========================================
% GOODREADS n=15 (K=5)
% ==========================================
\begin{figure*}[h]
    \centering
    \begin{subfigure}{0.9\textwidth}
    \centering
        \includegraphics[width=\textwidth]{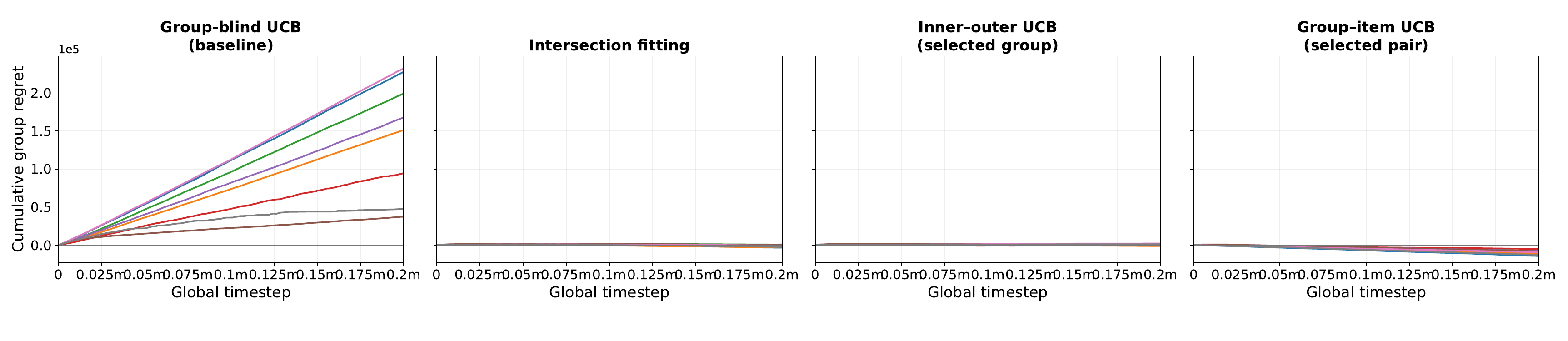}
        \caption{Base Clubs}
    \end{subfigure}
    % \vspace{0.2cm}
    
    \begin{subfigure}{0.9\textwidth}
    \centering
        \includegraphics[width=\textwidth]{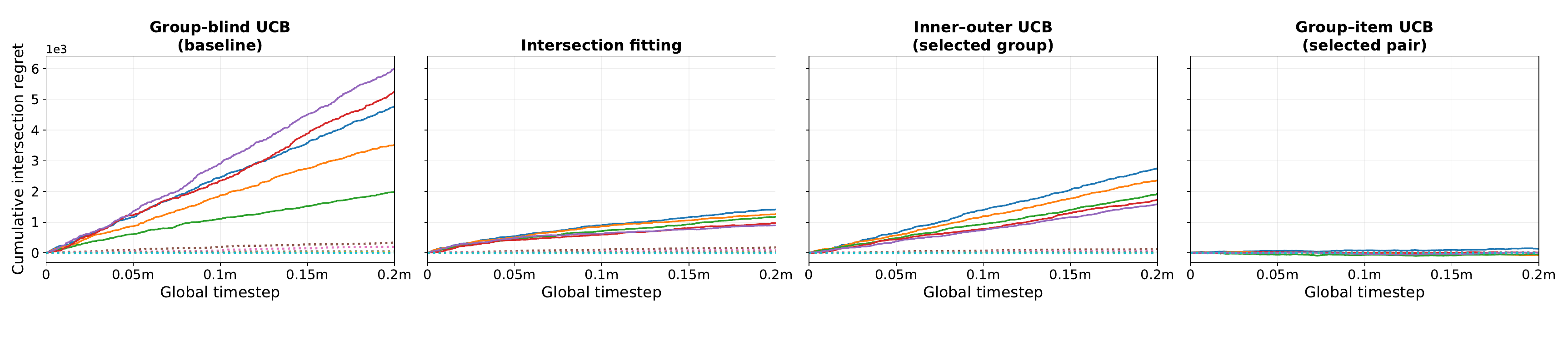}
        \caption{Common (Solid Lines) and Rare (Dotted Lines) Intersections}
    \end{subfigure}
    % \vspace{0.2cm}

    % \begin{subfigure}{\textwidth}
    %     \includegraphics[width=\textwidth]{outputs/goodreads/n=15/Goodreads_15items_200000rounds_K5_rare_regret.pdf}
    %     \caption{Rare Intersections}
    % \end{subfigure}

\caption{Goodreads primary catalog ($n=15$) at slate size \textbf{$K=5$}. Group-Blind UCB incurs linear regret across all panels. Group--Item UCB achieves the lowest regret across groups and intersections, whereas Inner--Outer UCB's cumulative regret exceeds Intersection-Fitting on common intersections.}
    \label{fig:goodreads_k5}
\end{figure*}

The group-based heuristics use this inherited structure, but their performance diverges at larger slates. Group--Item UCB achieves the lowest regret across the board: by evaluating joint (group, item) indices, it identifies optimal items efficiently and drives group regret near zero (and even slightly negative at $K=5$, as intersection-specific slates out-perform the group baseline).

Conversely, Inner--Outer UCB degrades at larger slate sizes. While efficient at $K=1$, its regret climbs as $K$ increases. At $K=3$ (Appendix A), Inner--Outer UCB still achieves comparable or lower regret than Intersection-Fitting on common intersections, but by $K=5$, it eventually loses to Intersection-Fitting at $T=200{,}000$. Because Inner--Outer UCB is restricted to selecting one group's entire proposed slate, it does not explicitly combine top indices across different groups the way Group--Item UCB does. This structural rigidity is a plausible penalty at $K=5$; even when data possesses inherited structure, pooling heuristics are not universally effective.

\subsection{BoardGameGeek: Performance with Heterogeneous Cell Preferences}
\label{subsec:bgg_results}

Unlike the Goodreads construction, the BoardGameGeek empirical calibration imposes no structural rule linking an intersection's preferences to its constituent groups. An exact intersection of users can have completely different optimal rankings than the broader themes to which they belong. 

\begin{figure*}[!t]
    \centering
    \begin{subfigure}{0.9\textwidth}
        \includegraphics[width=\textwidth]{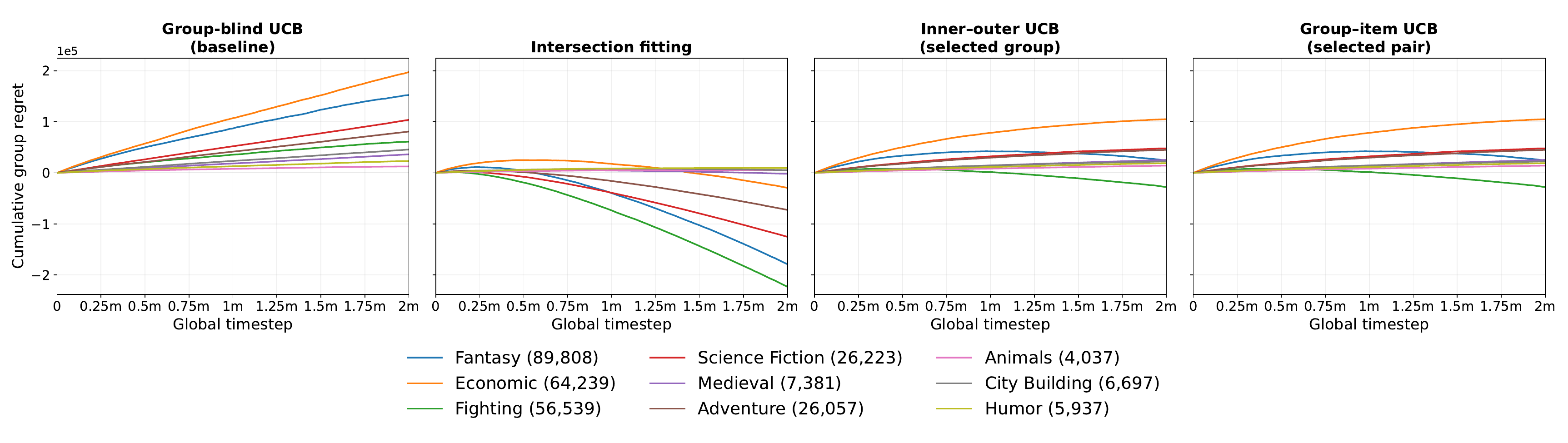}
        \caption{Base Themes: Group-level regret.}
    \end{subfigure}
    % \vspace{0.2cm}
    
    \begin{subfigure}{0.9\textwidth}
        \includegraphics[width=\textwidth]{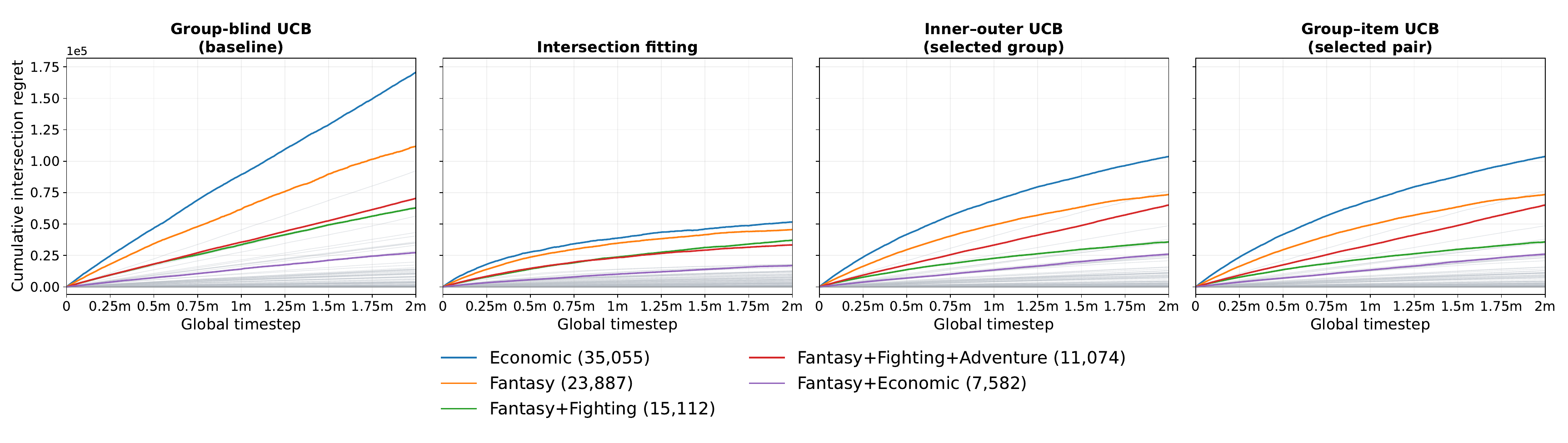}
        \caption{Intersections: Intersection-level regret.}
    \end{subfigure}
    
\caption{BoardGameGeek primary catalog ($n=15$) at slate size \textbf{$K=1$}. Intersection-Fitting achieves negative group-level regret by recommending intersection-specific slates that outperform the group's best fixed slate. Both pooling heuristics (Group--item UCB and Inner--Outer UCB) result in higher cumulative regret than Intersection-Fitting across the groups and intersections.}
    \label{fig:bgg_k1}
\end{figure*}

\begin{figure*}[!h]
    \centering
    \begin{subfigure}{0.9\textwidth}
        \includegraphics[width=\textwidth]{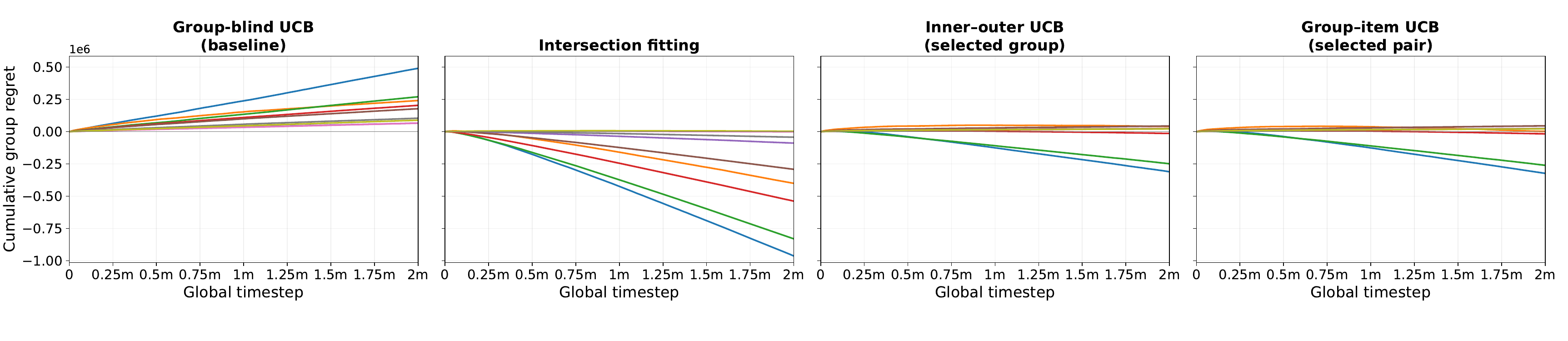}
        \caption{Base Themes: Group-level regret.}
    \end{subfigure}
    
    \begin{subfigure}{0.9\textwidth}
        \includegraphics[width=\textwidth]{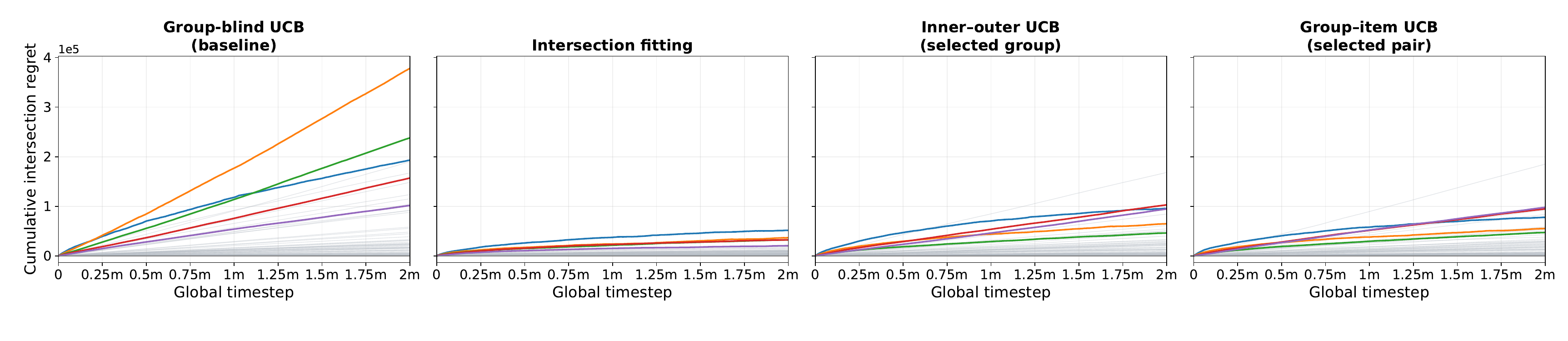}
        \caption{Intersections: Intersection-level regret.}
    \end{subfigure}
    
\caption{BoardGameGeek primary catalog ($n=15$) at slate size \textbf{$K=5$}. While the pooling heuristics also achieve negative group-level regret, Intersection-Fitting drives group regret to be more negative and achieves lower cumulative intersection regret across all displayed panels.}
    \label{fig:bgg_k5}
\end{figure*}

When intersections do not inherit this clean group structure, group-based heuristics lose their advantage. As shown across all slate sizes (with intermediate $K=3$ results demonstrating consistent performance trends deferred to Appendix~\ref{app:extended_graphs}), Inner--Outer UCB and Group--Item UCB fail to consistently outperform Intersection-Fitting. 
While pooling data across awake groups avoids maintaining an exponential state space, these results are consistent with harmful pooling across heterogeneous intersections: updating a pooled group model by whichever demographic intersection happened to be active systematically skews the UCB-style indices away from the optimal for a specific, atypical intersection. 
Consequently, Intersection-Fitting frequently matches or outperforms the heuristics by maintaining independent, unbiased estimates for each heterogeneous demographic intersection.

\subsection{Robustness to Randomness}

To ensure stability against the randomness of user arrivals and stochastic reward realizations, we evaluate all algorithms across five independent seeds. Across both the Goodreads and BoardGameGeek environments, the empirical regret concentrates around its expectation. At the final timestep, the relative standard deviation typically remains under $5\%$ of the cumulative regret magnitude for the baseline policies. When relative standard deviation appears mathematically large, the absolute standard deviation remains negligible (e.g., $\sigma < 35$ over $200{,}000$ rounds); this is the result of a near-zero regret denominator. Full details of the multi-seed variance analysis are deferred to Appendix~\ref{appendix:variance_robustness}.

\section{Discussion and Conclusion}
\label{sec:conclusion}
Group-blind recommendation (which excludes all current-user information in this model) fails structurally, not algorithmically. As demonstrated by our universal structural gap (Proposition~\ref{prop:universal}), no group-blind policy can overcome the penalty of optimizing for the "average" user, yielding strict linear total regret.

Our theoretical and empirical results yield three core takeaways for sequential recommendation. First, group-blind policies incur linear total regret: optimizing for the "average" user pays a linear welfare penalty, producing a $\Theta(T)$ total regret penalty for every group-blind policy at general $K$ under Assumption~\ref{ass:nondegenerate}. Second, while intersection-fitting is theoretically safe, it can be practically slow. Maintaining an independent learner per demographic intersection guarantees no-group regret, but the theoretical upper bound scales poorly with the number of intersections, which risks prohibitively slow finite-horizon learning. Finally, group-aware heuristics trade theoretical guarantees for empirical speed. We adapt two pooling heuristics (Inner--Outer and Group--Item UCB) that lack regret guarantees. When intersections inherit structured preferences from their base groups (Goodreads), pooling data can accelerate learning, though effectiveness varies: Group--Item UCB consistently improves sample efficiency, while Inner--Outer UCB's performance degrades at larger slate sizes. Furthermore, both heuristics fail to outperform exact intersection-fitting on BoardGameGeek, a contrast consistent with demographic intersections exhibiting heterogeneous properties from their constituent groups.

\paragraph{Limitations.} There are several limitations to our current analysis. 
First, while the universal structural gap holds for general $K$, positive theoretical guarantees cover only exact intersection-fitting. The pooling heuristics have no general upper guarantee; in fact, at $K=1$, both algorithms are mathematically equivalent and a constructed instance yields linear group regret (Appendix~\ref{app:counter}). Consequently, we rely on empirical evaluation for these shared-data architectures. Second, in general, no group-level algorithm, including ours, achieves sublinear total regret. Total regret decomposes as $\Rtot(T) = T(B_R^\mathcal{G} - B_{\mathrm{int}}) + \sum_\iota R_\iota(T)$. There remains an irreducible cost $T(B_R^\mathcal{G} - B_{\mathrm{int}})$ (zero in our experiments) associated with missing within-intersection information (i.e., not observing each user's exact, fully realized attribute vector). Finally, the structural gap we identify requires Assumption~\ref{ass:nondegenerate} to hold with a non-trivial margin. Real-world catalogs may have smaller preference divergence across demographic segments than our theoretical valuations.

\paragraph{Future work.} Our findings open several directions for future work. A primary technical extension would be identifying conditions under which pooled heuristics have theoretical guarantees, and tightening intersection fitting's dependence on $H_k$ toward the $\sqrt{\alpha_{k}T}$ lower bound. Finally, exploring strategic item response, continuous individual personalization, and hybrid architectures that use base-group pooling strictly as a cold-start prior until exact intersection counts clear the exploration bound---rather than treating group-specific and population-shared preferences as the two isolated extremes analyzed in Sections~\ref{sec:group_blind} and~\ref{sec:groupaware}---remain important next steps for modeling representation in sequential recommendation.

\section*{Acknowledgment}
Ziani gratefully acknowledges the support of the National Science Foundation through awards CAREER-2336236 and III-2504990.

%% Bibliography
\bibliographystyle{ACM-Reference-Format}
\bibliography{references} 

%% APPENDIX  (4 pages max)
\appendix
\input{arxiv_appendix}

\end{document}

%% file: arxiv_appendix.tex
\section{Additional Experimental Results}
\label{app:extended_graphs}
%\subsection{Omitted Slate Sizes ($n=15$)}
Figures~\ref{fig:goodreads_k3} and~\ref{fig:bgg_k3} present the $K=3$ slate evaluations for the Goodreads and BoardGameGeek catalogs, showing the same relative performance trends observed in the main text.

% ==========================================
% GOODREADS n=15 (K=3)
% ==========================================
\begin{figure*}[ht]
    \centering
    \begin{subfigure}{\textwidth}
        \includegraphics[width=\textwidth]{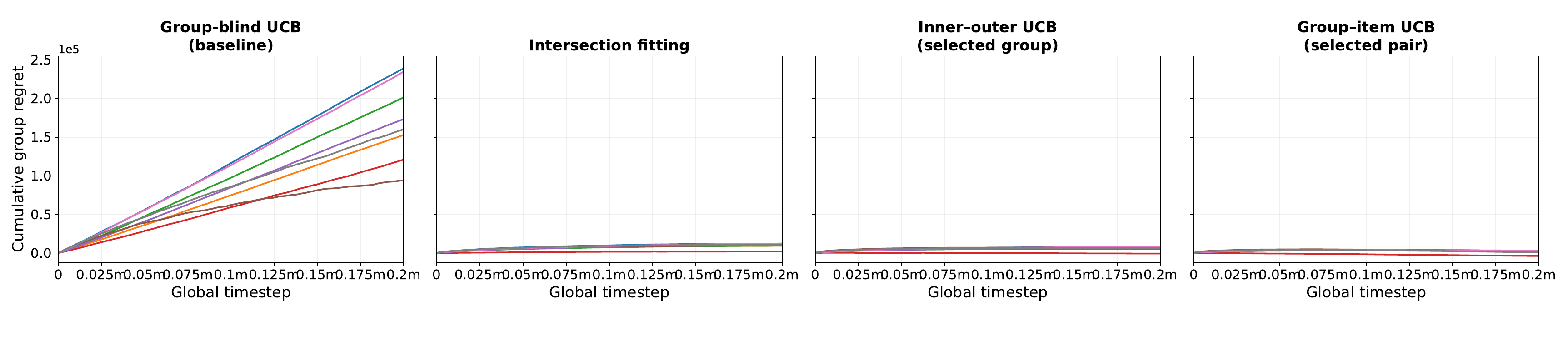}
        \caption{Base Clubs}
    \end{subfigure}
    % \vspace{0.2cm}
    
    \begin{subfigure}{\textwidth}
        \includegraphics[width=\textwidth]{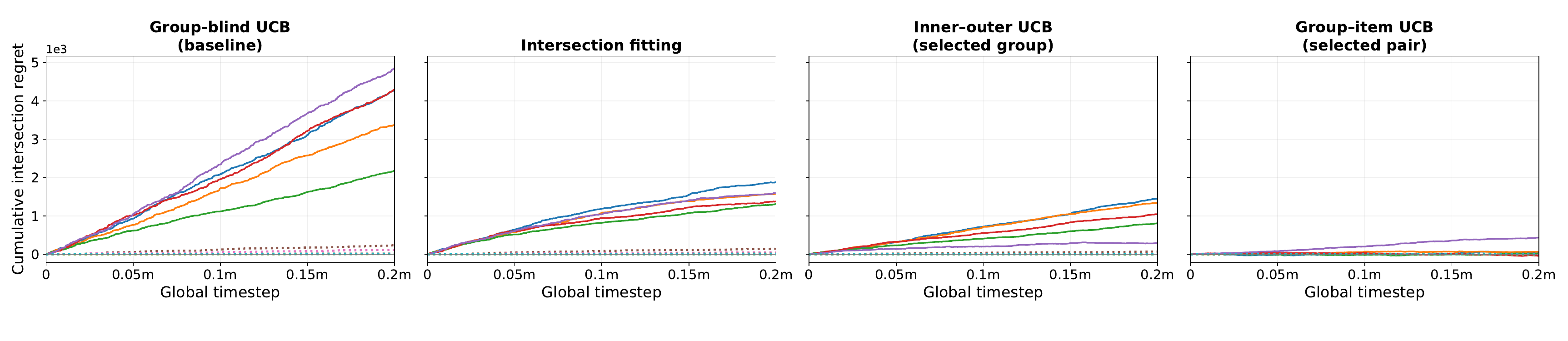}
        \caption{Common (Solid Lines) and Rare (Dotted Lines) Intersections}
    \end{subfigure}
    \caption{Goodreads primary catalog ($n=15$) at slate size \textbf{$K=3$}. On rare intersections (dotted lines), Group--item UCB and Inner--Outer UCB achieve lower regret than Intersection-Fitting. On common intersections (solid lines), Group--item UCB maintains the lowest regret, while Inner--Outer UCB achieves comparable cumulative regret to Intersection-Fitting.}
    \label{fig:goodreads_k3}
\end{figure*}

Figure~\ref{fig:bgg_k3} presents the corresponding $K=3$ evaluation for the primary BoardGameGeek catalog.

% ========================================== 
% BGG n=15 (K=3) 
% ==========================================
\begin{figure*}[ht]
    \centering
    \begin{subfigure}{\textwidth}
        \includegraphics[width=\textwidth]{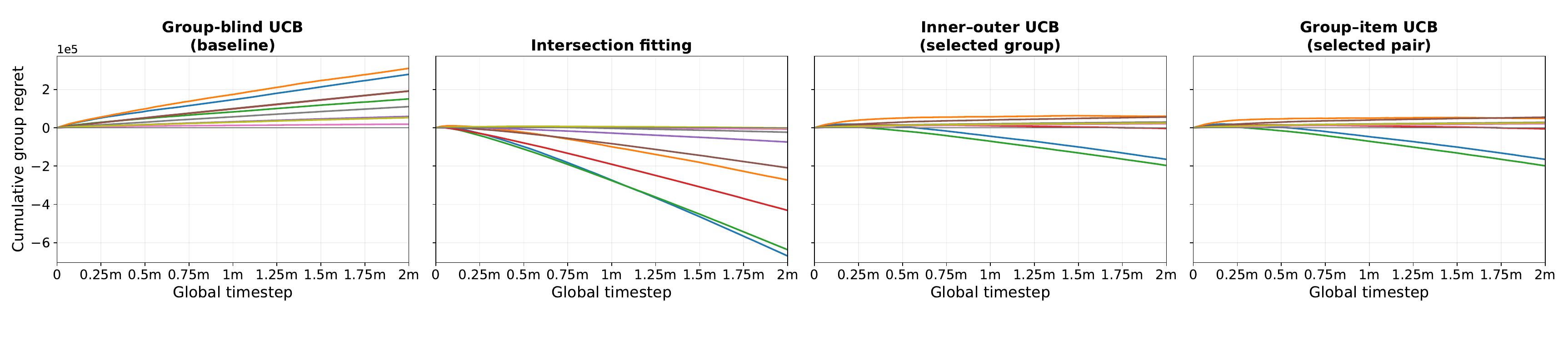}
        \caption{Base Themes: Group-level regret.}
    \end{subfigure}
    
    \begin{subfigure}{\textwidth}
        \includegraphics[width=\textwidth]{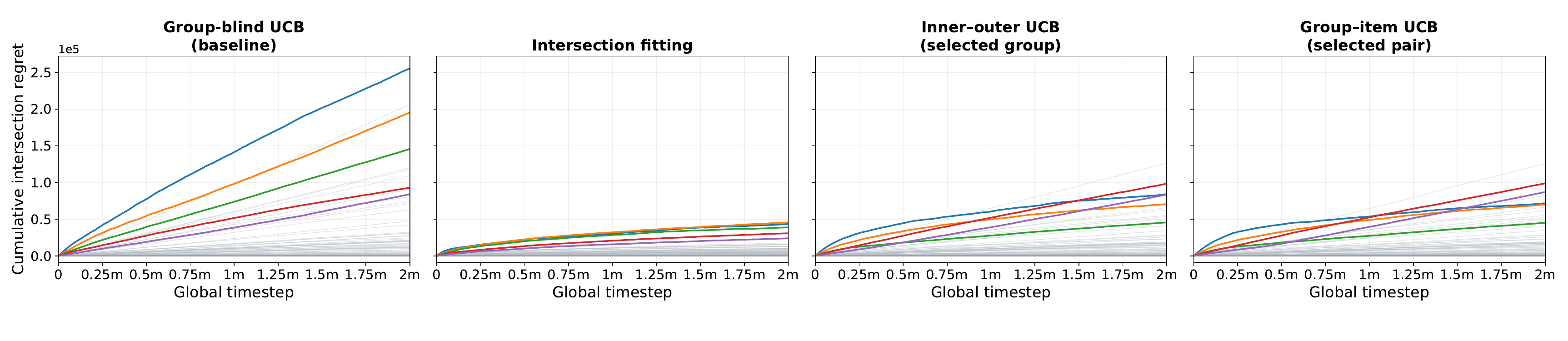}
        \caption{Intersections: Intersection-level regret.}
    \end{subfigure}
    
\caption{BoardGameGeek primary catalog ($n=15$) at slate size \textbf{$K=3$}. Intersection-Fitting achieves the lowest cumulative regret at both the group and intersection levels. Group--Item UCB and Inner--Outer UCB have distinct regret trajectories that end strictly higher than Intersection-Fitting across the panels.}
    \label{fig:bgg_k3}
\end{figure*}

\subsection{Extended Catalog Scaling ($n=72$)}
To verify that the dynamics observed in the primary $15$-item catalog hold at larger scales, we evaluate identical algorithmic setups on an extended catalog of $n=72$ items.

The core findings scale consistently, with one exception regarding Inner--Outer UCB. In the Goodreads environment (where intersections inherit group structure), Group--Item UCB continues to achieve lower regret than Intersection-Fitting across all $K$ (Figures~\ref{fig:goodreads_k1}, \ref{fig:goodreads_k3}, and \ref{fig:goodreads_k5}). While Inner--Outer UCB did not perform as well as Intersection-Fitting at $K=5$ on the $15$-item catalog, the larger $72$-item catalog reverses this ordering (Figure~\ref{fig:gr72_k5}). The expanded action space magnifies Intersection-Fitting's exploration fragmentation, allowing Inner--Outer UCB to maintain its advantage on common intersections.

In the BoardGameGeek environment (where intersections exhibit heterogeneous preferences), the penalty resulting from pooling still dominates. As in Figures~\ref{fig:bgg72_k1}, \ref{fig:bgg72_k3}, and \ref{fig:bgg72_k5}, Inner--Outer UCB and Group--Item UCB continue to suffer from the skewed UCB-style indices caused by active-group updates, allowing Intersection-Fitting to consistently match or outperform the heuristics exactly as it did at $n=15$.

\begin{figure*}[h]
    \centering
    \includegraphics[width=\textwidth]{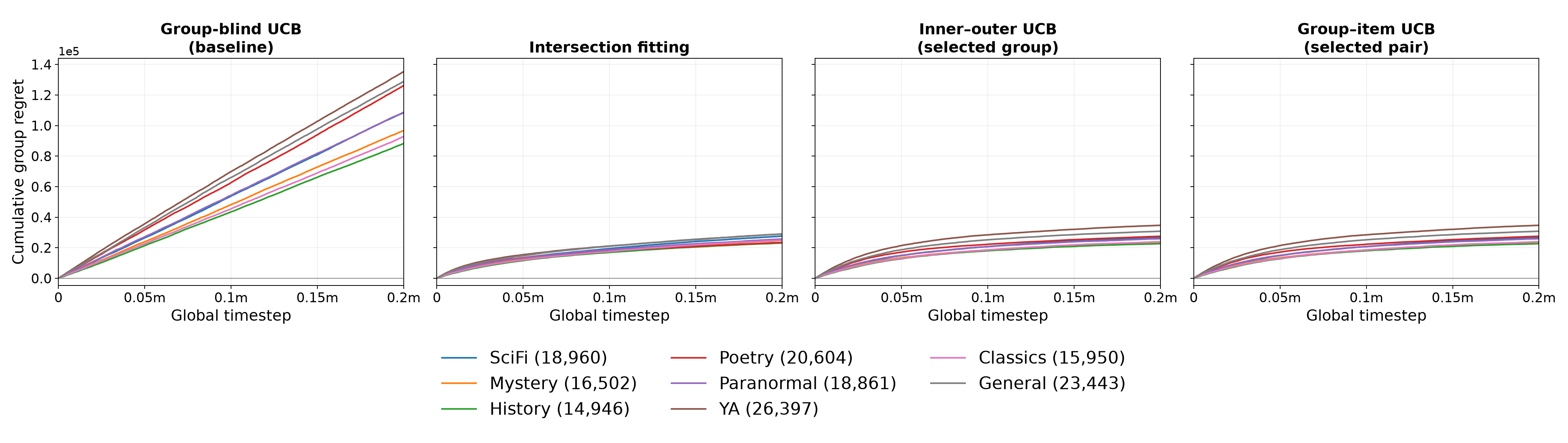}
    \includegraphics[width=\textwidth]{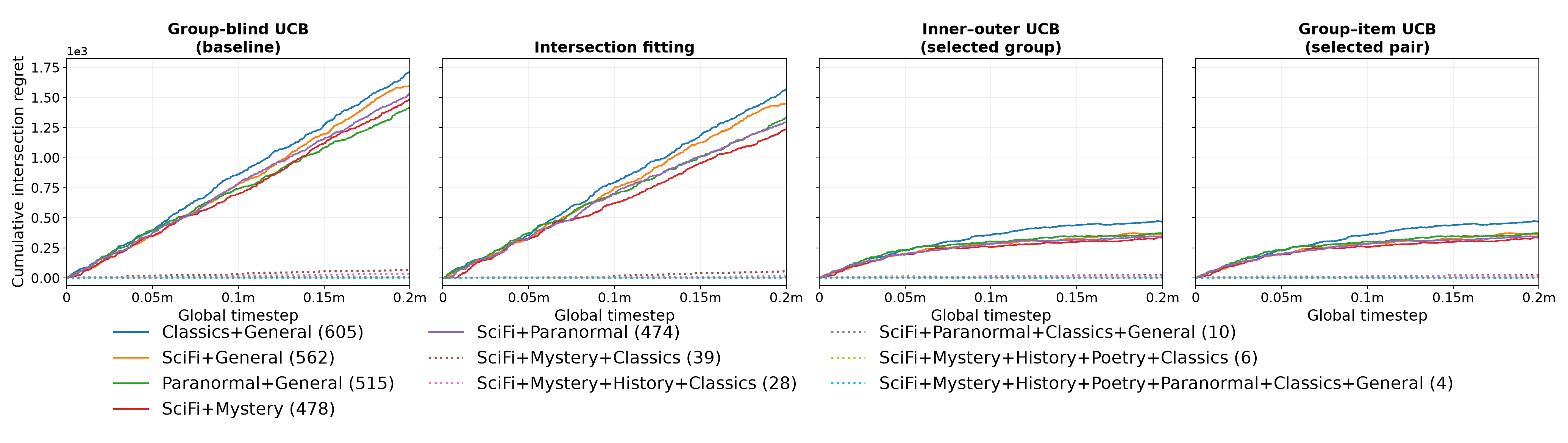}
    \caption{Goodreads extended catalog ($n=72$) at slate size $K=1$. At this slate size, Group--Item UCB and Inner--Outer UCB are algorithmically equivalent. Both pooled heuristics efficiently leverage the inherited group structure to achieve strictly lower cumulative regret than Intersection-Fitting across both group and intersection levels.}
    \label{fig:gr72_k1}
\end{figure*}

\begin{figure*}[h]
    \centering
    \includegraphics[width=\textwidth]{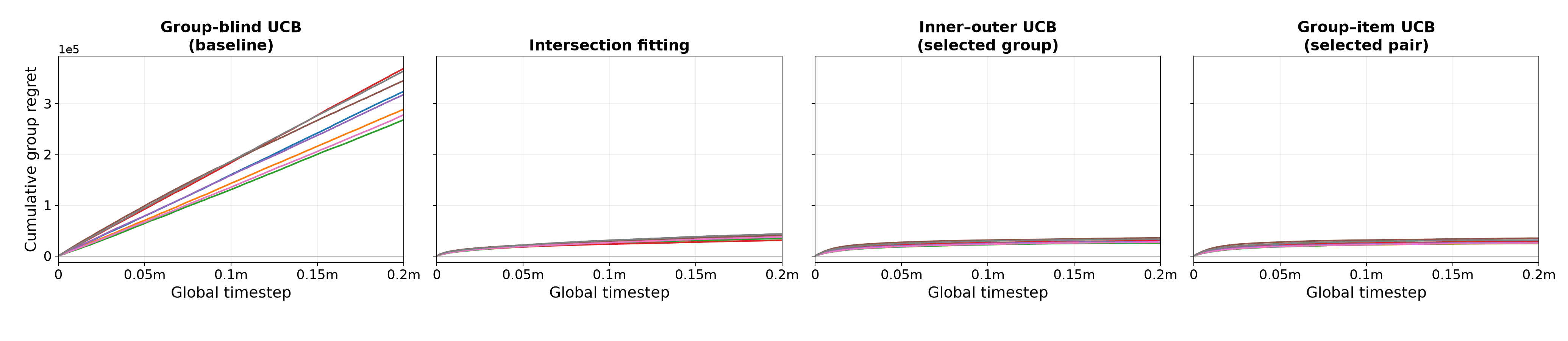}
    \includegraphics[width=\textwidth]{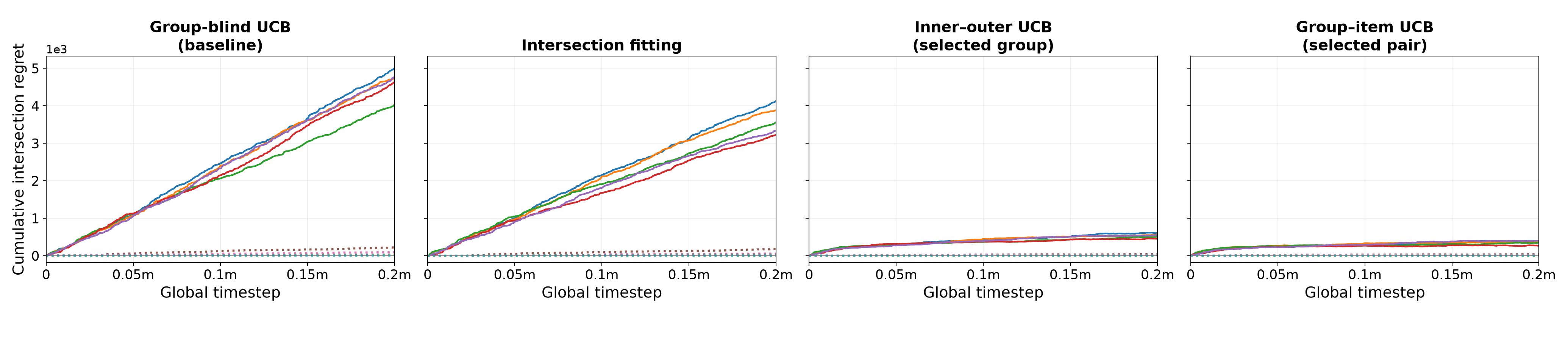}
    \caption{Goodreads extended catalog ($n=72$) at slate size $K=3$. Group--Item UCB achieves the lowest regret. Unlike the $15$-item catalog where Inner--Outer UCB and Intersection-Fitting were comparable at $K=3$, the expanded $72$-item catalog separates Intersection-Fitting's budget enough that Inner--Outer UCB establishes an advantage on common intersections.}
    \label{fig:gr72_k3}
\end{figure*}

\begin{figure*}[h]
    \centering
    \includegraphics[width=\textwidth]{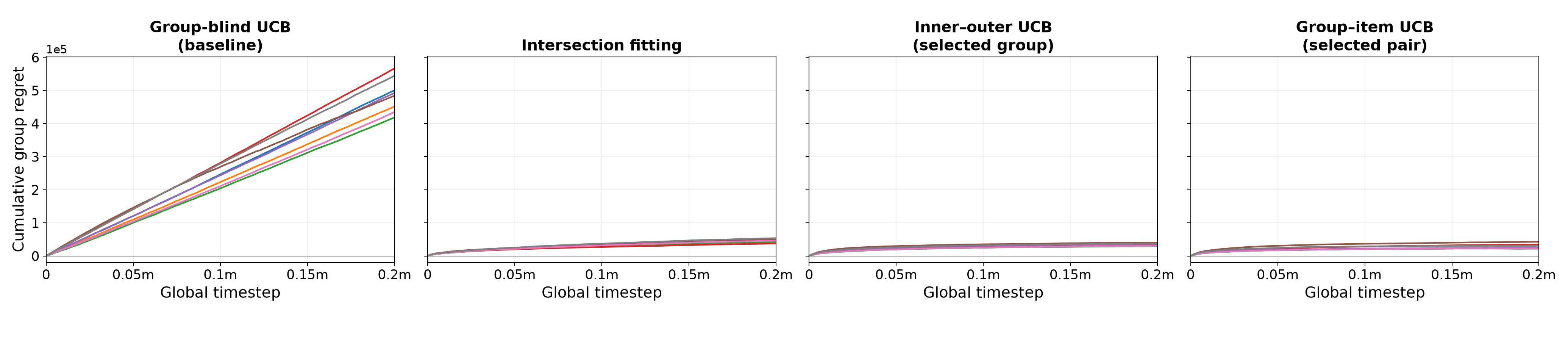}
    \includegraphics[width=\textwidth]{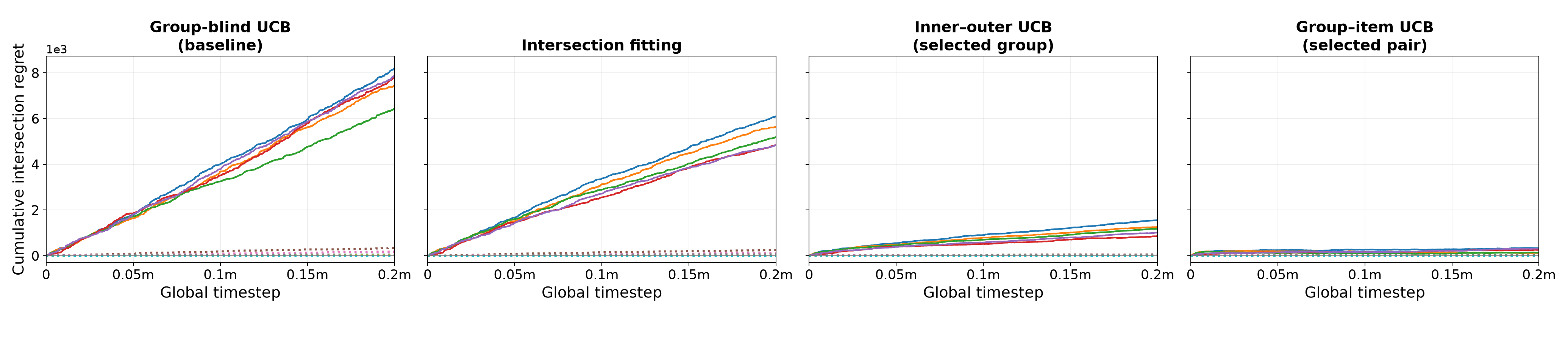}
    \caption{Goodreads extended catalog ($n=72$) at slate size $K=5$. This is a reversal of the dynamic observed in the $15$-item catalog: the expanded action space penalizes Intersection-Fitting's lack of data sharing more, allowing Inner--Outer UCB to maintain lower cumulative regret on common intersections. Group--Item UCB remains the most efficient policy.}
    \label{fig:gr72_k5}
\end{figure*}

\begin{figure*}[h]
    \centering
    \includegraphics[width=\textwidth]{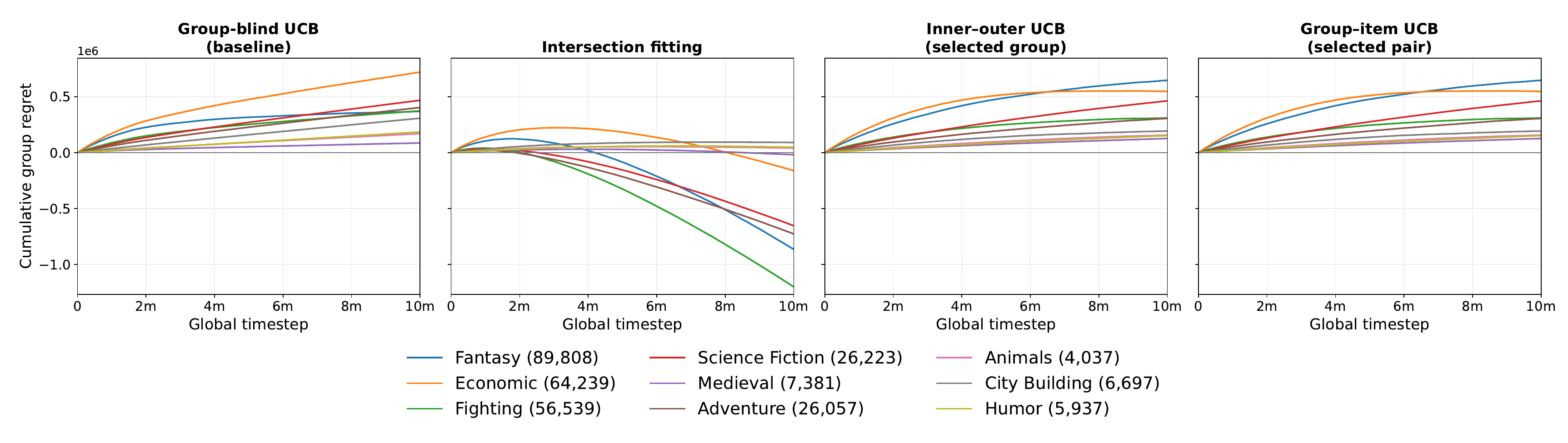}
    \includegraphics[width=\textwidth]{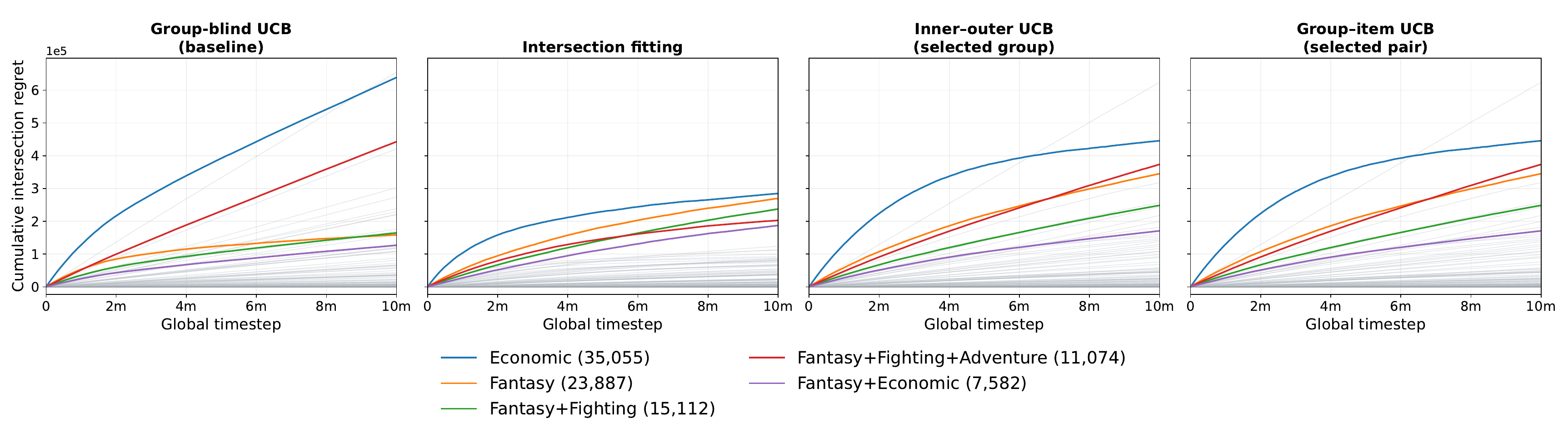}
    \caption{BoardGameGeek extended catalog ($n=72$) at slate size $K=1$. Because the empirical intersections exhibit heterogeneous preferences that do not smoothly compose from their base themes, pooling data actively harms learning. Intersection-Fitting avoids this cross-group interference, driving group regret deeply negative and achieving the lowest cumulative intersection regret.}
    \label{fig:bgg72_k1}
\end{figure*}

\begin{figure*}[h]
    \centering
    \includegraphics[width=\textwidth]{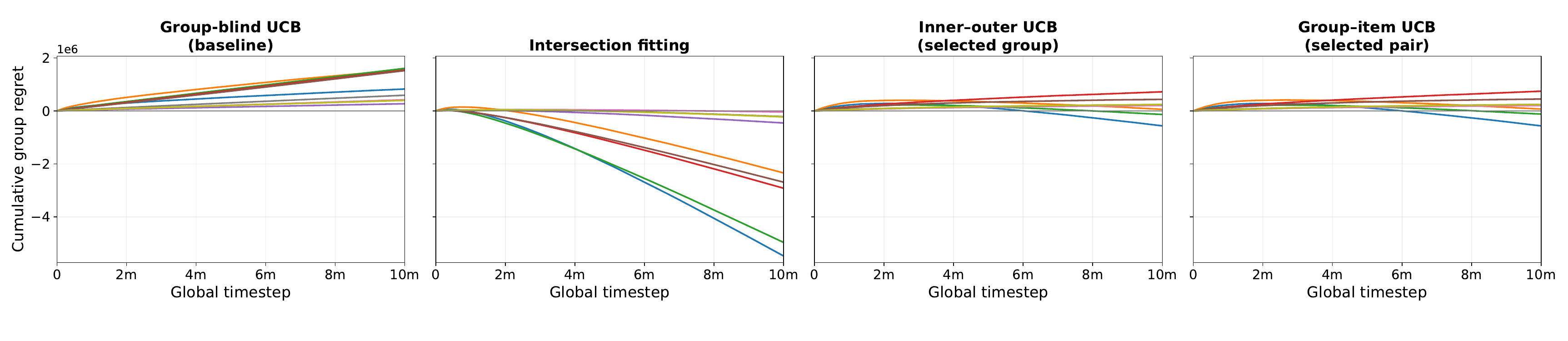}
    \includegraphics[width=\textwidth]{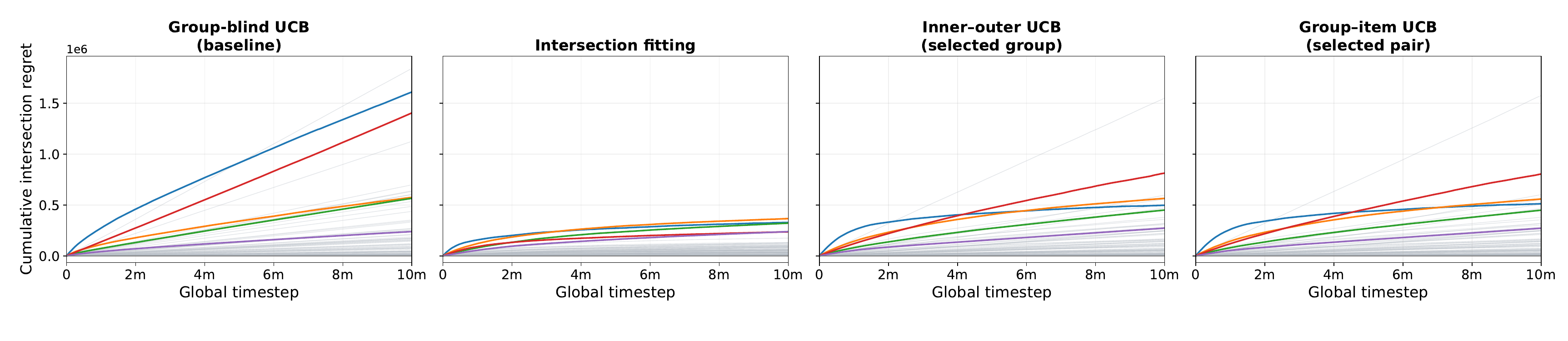}
    \caption{BoardGameGeek extended catalog ($n=72$) at slate size $K=3$. Intersection-Fitting consistently achieves the lowest cumulative regret across both evaluation levels. The skewed UCB-style indices caused by active-group updates prevent Inner--Outer UCB and Group--Item UCB from efficiently learning the optimal slates for atypical intersections.}
    \label{fig:bgg72_k3}
\end{figure*}

\begin{figure*}[h]
    \centering
    \includegraphics[width=\textwidth]{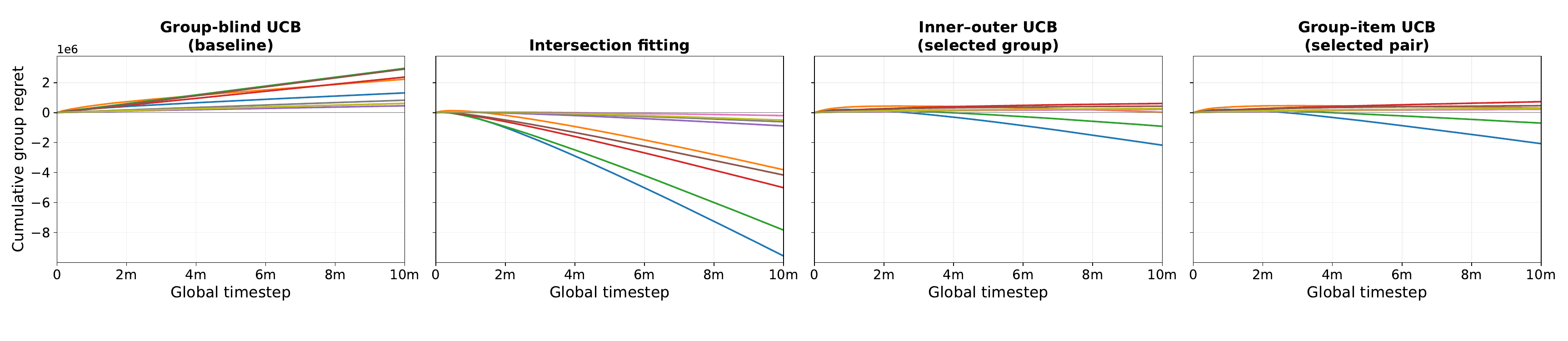}
    \includegraphics[width=\textwidth]{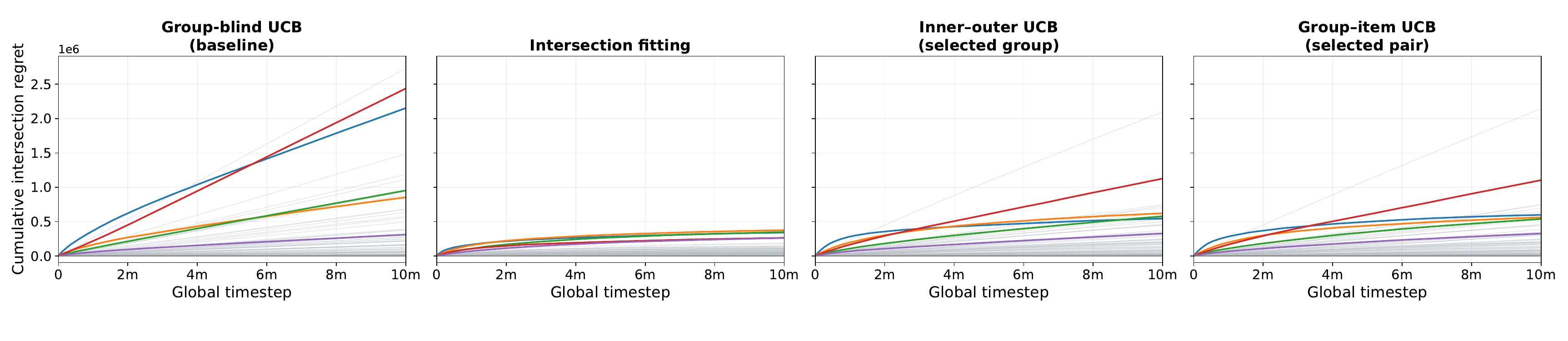}
    \caption{BoardGameGeek extended catalog ($n=72$) at slate size $K=5$. Even at larger slate sizes within the expanded catalog, Intersection-Fitting remains the dominant policy. The structural penalty of pooling heterogeneous data persists, allowing exact, unbiased intersection-level learning to outpace both heuristics.}
    \label{fig:bgg72_k5}
\end{figure*}

\section{Omitted Proofs}
\label{app:proofs}

\subsection{Proof of Theorem~\ref{thm:lower_bound}: Worst-Case Group Lower Bound}
\label{app:lower}
\begin{proof}
Use $K-1$ anchor items with deterministic reward $V_{\max}$; a candidate $u$ with reward $V_{\max}\cdot\mathrm{Bernoulli}(\tfrac12+\epsilon)$ in environment $+$ and $(\tfrac12-\epsilon)$ in $-$; and a candidate $w$ with $V_{\max}\cdot\mathrm{Bernoulli}(\tfrac12)$ in both. This uses $K+1\le n$ items. 

These laws apply on $g_k$-arrivals only; everything else, including arrivals, is identical across environments, and other items give $0$. The group-optimal slate is anchors$+u$ in $+$ and anchors$+w$ in $-$. Omitting $u$ in $+$ costs $\ge\epsilon V_{\max}$; including $u$ in $-$ costs $\ge\epsilon V_{\max}$ (serving both $u$ and $w$ drops an anchor and costs more). Since the user's scores are fresh, conditional on $S_t$ and an active round $\E[W_t]=\sum_{j\in S_t}v^*_{j,k}$, so $R_k$ is the expected slate loss.

For $\epsilon\le\tfrac14$, $\mathrm{kl}(\tfrac12+\epsilon,\tfrac12-\epsilon)=2\epsilon\ln\frac{1+2\epsilon}{1-2\epsilon}
\le\frac{8\epsilon^2}{1-2\epsilon}\le16\epsilon^2$. Only $u$'s score reveals the environment, and it contributes only on rounds where $g_k$ is active and $u$ is served, whose expected number is at most $\alpha_kT$; by the chain rule the KL of any history prefix is $\le16\epsilon^2\alpha_kT$. 

The fresh context has the same law in both environments. Bretagnolle--Huber \cite{lattimore2020bandit} on $\{u\in S_t\}$ at each active round gives $R_k^++R_k^-\ge\tfrac{\epsilon V_{\max}\alpha_kT}{2}e^{-16\epsilon^2\alpha_kT}$.

With $\epsilon=1/(4\sqrt{\alpha_kT})$ (valid since $\alpha_kT\ge1$) this is $\ge V_{\max}\sqrt{\alpha_kT}/(8e)$; the larger of $R^\pm_k$ is at least half of this.
\end{proof}

\subsection{Proof of Proposition~\ref{prop:isfit}: Intersection Fitting}
\label{app:isfit}
\begin{proof}
Fix one intersection $\iota$, and condition on its number of arrivals being $s$. The local reward streams are independent across arrivals, bounded in $[0,V_{\max}]$, and have item means $\mu_{\iota,j}$. Algorithm 2 sees no information within the intersection beyond these observations. Divide rewards by $V_{\max}$ for the proof, and write the normalized means as $\mu_j \in [0,1]$. The count clock $\tau$ below counts arrivals up to and including the current one.

Let $S^*$ be an optimal fixed $K$-slate in this intersection. At a local round with all items observed, define
\[ r_{j,\tau} = \sqrt{\frac{2 \log \max\{2,\tau\}}{N_{j,\tau}}}, \qquad U_{j,\tau} = \hat{\mu}_{j,\tau} + r_{j,\tau}. \]
On the event that every item's estimate is within its radius, optimism and the top-$K$ rule give
\[ \sum_{j \in S^*} \mu_j - \sum_{j \in S_\tau} \mu_j \le \sum_{j \in S^*} U_{j,\tau} - \sum_{j \in S_\tau} \mu_j \le \sum_{j \in S_\tau} (U_{j,\tau} - \mu_j) \le 2 \sum_{j \in S_\tau} r_{j,\tau}. \]

Any round with an unobserved item selected incurs regret at most $K$ and observes at least one new item. There are at most $n$ such rounds, so their total regret is at most $nK$. If any item is unobserved, the top-$K$ infinite-index rule selects at least one unobserved item; thus every remaining round has finite indices for all items.

For each item, expose an infinite i.i.d. sequence of its rewards on the intersection arrivals where it is selected. A union bound over all possible sample sizes $q \le \tau$ controls the adaptive counts via Hoeffding's inequality:
\[ \Pr(\text{some radius fails at local clock } \tau) \le \frac{2n\tau}{\max\{2,\tau\}^4} \quad (\tau \ge 1). \]
Independence across items is not needed for this union bound. Moreover,
\begin{align*}
    \sum_{\tau=1}^\infty \frac{2n\tau}{\max\{2,\tau\}^4} = 2n\left(\frac{1}{16} + \sum_{\tau=2}^\infty \tau^{-3}\right) \le 2n\left(\frac{1}{16} + \frac{1}{8} + \int_2^\infty x^{-3} dx\right) = \frac{5}{8}n \le n.
\end{align*}
Since per-round regret is at most $K$, all failed-confidence rounds contribute at most $nK$ in expectation. This deliberately loose bound leaves the stated $2nK$ additive term after initialization.

Let $N_j(s)$ denote item $j$'s total pulls in these $s$ arrivals. Summing the good-round confidence terms, and allowing extra nonnegative terms when convenient, gives
\begin{align*}
    \sum_{\text{good, initialized rounds}} 2 \sum_{j \in S_\tau} r_{j,\tau} \le 2 \sqrt{2 \log(2s)} \sum_{j=1}^n \sum_{q=1}^{N_j(s)-1} q^{-1/2} \le 4 \sqrt{2 \log(2s)} \sum_{j=1}^n \sqrt{N_j(s)} \le 4 \sqrt{2nKs \log(2s)},
\end{align*}
where empty sums are zero and $\sum_j N_j(s) = Ks$. Rescaling establishes, for $s > 1$,
\[ \E[\text{intersection regret} \mid T_\iota = s] \le V_{\max}\left[4\sqrt{2nKs \log(2s)} + 2nK\right]. \]
The regret is zero at $s=0$. Since $T_\iota \le T$, $\E[T_\iota] = p_\iota T$, and $\E[\sqrt{T_\iota}] \le \sqrt{p_\iota T}$, Jensen's inequality yields
\[ R_\iota(T) \le V_{\max}\left[4\sqrt{2nKp_\iota T \log(2T)} + 2nK\right]. \]
Conditioning on arrival counts is legitimate here because the local algorithm uses its intersection clock and the conditional reward law remains unchanged. For the group step, define
\[ b_\iota = \max_{|S|=K} \sum_{j \in S} \mu_{\iota,j}, \qquad b_k = \max_{|S|=K} \sum_{\iota \ni k} \frac{p_\iota}{\alpha_k} \sum_{j \in S} \mu_{\iota,j}. \]
The maximum of an average is at most the average of maxima, so
\[ \Gamma_k := \sum_{\iota \ni k} p_\iota b_\iota - \alpha_k b_k \ge 0. \]
Separating group arrivals into intersections yields the exact identity
\[ R_k(T) = \sum_{\iota \ni k} R_\iota(T) - T \Gamma_k \le \sum_{\iota \ni k} R_\iota(T). \]
It is an inequality, not equality between group regret and the sum of intersection regrets. Consequently,
\begin{align*}
    R_k(T) & \le V_{\max}\left[4\sqrt{2nKT \log(2T)} \sum_{\iota \ni k} \sqrt{p_\iota} + 2nKH_k\right] \\ & \le V_{\max}\left[4\sqrt{2nKH_k\alpha_k T \log(2T)} + 2nKH_k\right]
\end{align*}
using Cauchy-Schwarz and $\sum_{\iota \ni k} p_\iota = \alpha_k$. This is the stated bound.
\end{proof}

\subsection{Counterexample for Group-Based Heuristics}
\label{app:counter}
At $K=1$, Algorithms~\ref{alg:hier} and~\ref{alg:flat} are equivalent and can suffer linear group regret. Consider the universal group $0$ and base groups $A, B$ with deterministic rewards ($V_{\max}=1$). Let A-only users (mass $0.8$) have $(v_{ix},v_{iy})=(1,0)$; $A \cap B$ users (mass $0.1$) have $(v_{ix},v_{iy})=(0,1/2)$; and B-only users (mass $0.1$) have $(v_{ix},v_{iy})=(0,3/5)$. 

We bound the asymptotic group regret for $B$. The best fixed item for group $B$ is $y$, providing an unconditional per-round benchmark contribution of $0.1(1/2) + 0.1(3/5) = 0.11$.

First, we show the pair $(A, x)$ is selected infinitely often. Under winning-pair updates with a global clock and fixed tie order, every pair awake on A-only arrivals must have diverging counts. If any such pair maintained a finite limiting count, its index divided by $\sqrt{\log t}$ would remain bounded away from zero, whereas the ratio for diverging-count pairs tends to zero. Consequently, a bounded-count pair would eventually win every A-only arrival. Because there are finitely many pairs and infinitely many A-only arrivals, this is a contradiction. Thus, the count for $(A, x)$ diverges.

Second, we bound the sampling bias. Conditional on the past, if $(A, x)$ wins on an overlap ($A \cap B$) arrival, it also wins on an A-only arrival, since the latter strictly removes competitors while preserving tie order. Thus, whenever $(A, x)$ is selected, the conditional probability of an A-only arrival is at least $0.8 / (0.8 + 0.1) = 8/9$. By a bounded-increment martingale strong law along its sample times, its empirical mean satisfies $\liminf_{t\to\infty} \hat{v}_{x,A,t} \ge 8/9$ almost surely.

Third, we bound the overlap plays of item $y$. Fix a constant $a \in (3/5, 8/9)$. After an almost surely finite random time, $I_{x,A,t} \ge a$. The empirical mean of any $y$ pair is at most $3/5$. If a $y$ pair wins on an overlap round after this time, its pre-update count must satisfy $N_{y,k,t} \le \frac{2 \log \max\{2,t+1\}}{(a-3/5)^2}$. Each such win increments its count, so summing over the three possible owners of $y$ yields only $O(\log T)$ such plays. While another pair might beat $(A, x)$, any such pair involving $x$ yields zero reward on the overlap. 

Finally, group $B$ (mass $0.2$) earns at most $3/5$ on its B-only arrivals (mass $0.1$) and $0$ asymptotically on the overlap (mass $0.1$). This yields a total asymptotic reward of at most $0.06T + o(T)$ almost surely. Comparing this to $B$'s benchmark of $0.11T$, the pathwise shortfall divided by $T$ has a limit inferior of at least $0.05$. By boundedness and Fatou's lemma, the expected group regret satisfies:
\[ \liminf_{T\to\infty} \frac{R_B(T)}{T} \ge 0.05. \]

\section{Simulation and Calibration Details}
\label{app:simulation}
To evaluate our architecture, we require environments where both the group topologies and the intersectional item valuations are rigorously defined. We detail the calibration methodologies for our synthetic-valuation and empirical-rating datasets below.

\subsection{Goodreads: Lancaster Characteristics Calibration} 
The Goodreads dataset isolates the effect of real-world group topologies by pairing the observed club co-membership network with a controlled Lancaster characteristics model \cite{lancaster1966new}. We extract the membership patterns for $G=8$ primary clubs (e.g., Sci-Fi, Mystery, History) and construct a catalog of $n=72$ books, allocating nine per genre. Each book $j$ is defined by a feature vector $x_j \in \mathbb{R}^G$ dictating its characteristic weights. To guarantee variation and prevent exact ranking ties within $K$-slates, the nine books in genre $k$ are distributed across three purity tiers using $\pm 0.04$ internal feature offsets to ensure distinct item valuations: Core Flagship (3 items, primary weight $0.88$), Deep Hybrid (3 items, primary weight $0.68$, pre-defined adjacent genre $0.22$), and Crossover Hybrid (3 items, primary weight $0.50$, second adjacent genre $0.38$). Across all tiers, the remaining weight is distributed uniformly as diffuse background noise.

% The Goodreads dataset isolates the effect of real-world group topologies by pairing the observed club co-membership network with a controlled Lancaster characteristics model \cite{lancaster1966new}. We extract the membership patterns for $G=8$ primary clubs (e.g., Sci-Fi, Mystery, History).

% We construct a catalog of $n=72$ books, allocating nine books per genre. Rather than treating items as monolithic entities, each book $j$ is defined by a feature vector $x_j \in \mathbb{R}^G$ dictating its characteristic weights. To guarantee variation and prevent exact ranking ties within $K$-slates, the nine books in genre $k$ are distributed across three purity tiers, using $\pm 0.04$ internal feature offsets to ensure distinct item valuations:
% \begin{itemize}
%     \item Core Flagship (3 items): Primary weight centered at $0.88$, with the remainder distributed uniformly as diffuse background noise.
%     \item Deep Hybrid (3 items): Primary weight centered at $0.68$, a pre-defined adjacent genre at $0.22$, and background noise.
%     \item Crossover Hybrid (3 items): Primary weight centered at $0.50$, a second adjacent genre at $0.38$, and background noise.
% \end{itemize}

User base tastes are modeled as a base centroid matrix $\Theta \in \mathbb{R}^{G \times G}$, defined as the identity matrix scaled by $0.85$ plus $0.15$ uniform background taste. For a user belonging to demographic intersection $\iota$, their composite taste vector $\bar{\theta}_\iota$ is the mean of the centroids of their active clubs. The centroids are generative parameters, not necessarily the pooled conditional means.

The raw latent utility for an item is the dot product $u_{\iota,j} = x_j \cdot \bar{\theta}_\iota$. To map these latent utilities to Gaussian centers $\nu_{\iota,j}$, we apply an exact linear scaling, bounded to our valuation range:
\[ \nu_{\iota,j} = \text{clip}_{[1.0, 8.0]}\left(2.5 + 5.0 \frac{u_{\iota,j}-0.05}{0.70}\right) \]
Simulated noise and final reward clipping are applied after this structural scaling. During the simulation, a user's realized welfare reward $Y_t$ for recommended item $j$ is drawn as $Y_t = \text{clip}_{[r_{\min}, r_{\max}]}(\mathcal{N}(\nu_{\iota,j}, \sigma^2))$ with $[r_{\min}, r_{\max}] = [1, 8]$ (setting $V_{\max} = r_{\max}$) and $\sigma = 1.0$. Benchmarks use the exact post-clipping mean:
\begin{align*}
    v_{j,\iota_P}^* = m(\nu_{\iota,j}, \sigma) = r_{\min}\Phi(z_{lo}) + r_{\max}(1-\Phi(z_{hi})) + \nu_{\iota,j}(\Phi(z_{hi}) - \Phi(z_{lo})) + \sigma(\phi(z_{lo}) - \phi(z_{hi})),
\end{align*}
where $z_{lo} = (r_{\min}-\nu_{\iota,j})/\sigma$ and $z_{hi} = (r_{\max}-\nu_{\iota,j})/\sigma$, and $\Phi, \phi$ are the standard normal cdf and pdf.

\subsection{BoardGameGeek (BGG): Empirical Rating Calibration}
The BGG environment uses empirical ratings and inferred group structure: the empirical mean of observed ratings determines the Gaussian center for each (intersection, item) pair, and only the per-round noise around that mean is simulated as draws. 
% The BGG environment uses empirical ratings and inferred group structure: the expected welfare of each (intersection, item) pair is the empirical mean of observed ratings, and only the per-round noise around that mean is simulated as draws.
We filter the dataset to catalogs of games with at least 500 reviews across $G=9$ overarching themes. A user is defined as active in theme $g_k$ if their rating history contains at least five games tagged with that theme, comprising $\ge 20\%$ of their total rating history.

Because our objective is to maximize absolute user welfare, we use the raw, unadjusted user ratings (from 1.0 to 10.0) directly as the reward signal. However, using raw ratings on globally popular games introduces a structural hazard: universally liked games act as ``universal sinks,'' dominating the optimal slates for nearly every intersection. To construct an environment that is designed to reward personalization without altering the underlying welfare metric, we filter the catalog to explicitly remove universally liked items while preserving pure preference ordering. We isolate item-theme pairs with at least 20 ratings to ensure statistical reliability. If an item ranks in the top 15\% of at least half of the themes in which it is evaluated (with no minimum theme count floor), it is designated a universal sink and excluded. Each remaining game is exclusively assigned to the theme where it achieves its highest percentile rank. Finally, we extract the top 8 polarizing games per theme, ranked strictly by their raw empirical mean. This yields a deduplicated catalog of $n=72$ items, designed to increase preference variation.

However, this pooled vector can hide different optimal $K$-slates within $g_k$. The Gaussian center $\nu_{\iota,j}$ of each (intersection, item) pair is the empirical mean of observed ratings, and the raw intersection-item variance $\hat{\sigma}_{\iota,j}^2$ is calculated strictly from the observed ratings submitted by users belonging to that exact intersection. Due to clipping bounds, these Gaussian centers are distinct from the final post-clipping expected rewards used in the benchmarks. We apply a strict insufficient-data threshold of zero: if an intersection contains zero observed ratings for a specific item, we default to the item's global empirical mean and standard deviation.

During the simulation, a user's realized welfare reward $Y_t$ for recommended item $j$ is drawn as $\text{clip}_{[r_{\min},r_{\max}]}(\mathcal{N}(\nu_{\iota,j}, \sigma_{\iota,j}^2))$ with $[r_{\min},r_{\max}]=[1,10]$ and a standard deviation floor $\sigma_{\iota,j} := \max\{\hat{\sigma}_{\iota,j}, 0.5\}$ to prevent degenerate simulated reward distributions. Because raw ratings are used directly as the reward signal, no normalization is applied. Expected welfare is its post-clipping mean $m(\nu_{\iota,j},\sigma_{\iota,j})$, which the regret benchmarks use exactly:
\begin{align*}
    v_{j,\iota_P}^* = m(\nu_{\iota,j}, \sigma_{\iota,j}) = r_{\min}\Phi(z_{lo}) + r_{\max}(1-\Phi(z_{hi})) + \nu_{\iota,j}(\Phi(z_{hi})
    - \Phi(z_{lo})) + \sigma_{\iota,j}(\phi(z_{lo}) - \phi(z_{hi})), 
\end{align*}
where $z_{lo} = (r_{\min}-\nu_{\iota,j})/\sigma_{\iota,j}$ and $z_{hi} = (r_{\max}-\nu_{\iota,j})/\sigma_{\iota,j}$.

\section{Robustness to Variance}
\label{appendix:variance_robustness}
To verify the robustness of our experiments, we evaluate algorithmic variance across multiple independent simulation paths. We run the simulations across $5$ independent random seeds for both the Goodreads ($T = 200{,}000$) and BoardGameGeek ($T \in \{2 \times 10^6, 10^7\}$) environments, covering catalog sizes $n \in \{15, 72\}$ and slate sizes $K \in \{1, 3, 5\}$.

At each global timestep, we compute the absolute standard deviation $\sigma$ across the $5$ independent realizations. To evaluate relative dispersion, we renormalize this standard deviation by the mean cumulative regret magnitude. Table~\ref{tab:goodreads_variance} and Table~\ref{tab:bgg_variance} detail these final-timestep metrics for the Goodreads and BoardGameGeek simulations, respectively.

% We omit graphical confidence intervals from the primary regret plots because they would be visually indistinguishable from the mean trajectories. This is a result of the evaluation scale: because the evaluated algorithms exhibit linear cumulative regret, the visual $y$-axis scales linearly with $T$. In contrast, the absolute standard deviation of the cumulative regret grows at a sublinear rate of $\mathcal{O}(\sqrt{T})$.

Our variance analysis shows that at the final timestep, the relative standard deviation remains small---typically under $5\%$ of the total regret magnitude across evaluated policies for both environments. When the relative standard deviation appears mathematically large---specifically when Group--Item UCB drives intersection regret near zero in the Goodreads environment---this is the result of a vanishing denominator. The absolute standard deviation in these instances remains small (e.g., $\sigma < 35$ over $200{,}000$ rounds). Similarly, while absolute standard deviations in BoardGameGeek scale numerically higher due to the extended horizons, they remain proportionally bounded below $4\%$ of the realized regret. 
% Because theoretical error bounds would occupy sub-pixel vertical space within the plotted line widths, this tabular summary confirms that 
The empirical regret concentrates around its expectation, so the reported results are representative of expected algorithmic behavior.

\begin{table*}[htbp]
\centering
\caption{Robustness Analysis: Multi-Seed Variance in Goodreads ($T=200{,}000$)}
\label{tab:goodreads_variance}
\resizebox{\textwidth}{!}{
\begin{tabular}{ccclrrr}
\toprule
\textbf{Catalog} & \textbf{$K$} & \textbf{Level} & \textbf{Algorithm} & \textbf{Avg Final $|\text{Regret}|$} & \textbf{Avg Std Dev} & \textbf{Relative Std} \\
\midrule
$n=15$ & 1 & Group & Group-Blind & 97,102.9 & 863.0 & 0.8888\% \\
$n=15$ & 1 & Group & Intersection-Fit & 9,105.6 & 203.3 & 2.2331\% \\
$n=15$ & 1 & Group & Inner-Outer & 5,616.8 & 166.0 & 2.9552\% \\
$n=15$ & 1 & Group & Group-Item & 5,616.8 & 166.0 & 2.9552\% \\
$n=15$ & 1 & Intersection & Group-Blind & 728.4 & 34.5 & 4.7325\% \\
$n=15$ & 1 & Intersection & Intersection-Fit & 513.8 & 21.6 & 4.1955\% \\
$n=15$ & 1 & Intersection & Inner-Outer & 44.1 & 13.2 & 29.9056\% \\
$n=15$ & 1 & Intersection & Group-Item & 44.1 & 13.2 & 29.9056\% \\
\midrule
$n=15$ & 3 & Group & Group-Blind & 171,902.3 & 1,880.9 & 1.0942\% \\
$n=15$ & 3 & Group & Intersection-Fit & 9,519.7 & 303.8 & 3.1918\% \\
$n=15$ & 3 & Group & Inner-Outer & 6,228.3 & 300.5 & 4.8252\% \\
$n=15$ & 3 & Group & Group-Item & 2,432.3 & 236.6 & 9.7254\% \\
$n=15$ & 3 & Intersection & Group-Blind & 1,970.1 & 99.6 & 5.0559\% \\
$n=15$ & 3 & Intersection & Intersection-Fit & 813.2 & 39.8 & 4.8927\% \\
$n=15$ & 3 & Intersection & Inner-Outer & 515.9 & 34.3 & 6.6524\% \\
$n=15$ & 3 & Intersection & Group-Item & 54.1 & 26.5 & 48.9433\% \\
\midrule
$n=15$ & 5 & Group & Group-Blind & 144,763.9 & 2,008.5 & 1.3874\% \\
$n=15$ & 5 & Group & Intersection-Fit & 1,290.1 & 340.7 & 26.4127\% \\
$n=15$ & 5 & Group & Inner-Outer & 1,351.4 & 330.5 & 24.4585\% \\
$n=15$ & 5 & Group & Group-Item & 9,253.8 & 323.7 & 3.4978\% \\
$n=15$ & 5 & Intersection & Group-Blind & 2,250.1 & 111.7 & 4.9656\% \\
$n=15$ & 5 & Intersection & Intersection-Fit & 614.8 & 45.1 & 7.3310\% \\
$n=15$ & 5 & Intersection & Inner-Outer & 1,097.1 & 64.8 & 5.9039\% \\
$n=15$ & 5 & Intersection & Group-Item & 13.2 & 31.8 & 239.7872\% \\
\midrule
$n=72$ & 1 & Group & Group-Blind & 110,697.6 & 677.5 & 0.6121\% \\
$n=72$ & 1 & Group & Intersection-Fit & 25,970.6 & 224.6 & 0.8649\% \\
$n=72$ & 1 & Group & Inner-Outer & 26,950.3 & 241.2 & 0.8949\% \\
$n=72$ & 1 & Group & Group-Item & 26,950.3 & 241.2 & 0.8949\% \\
$n=72$ & 1 & Intersection & Group-Blind & 797.3 & 36.3 & 4.5541\% \\
$n=72$ & 1 & Intersection & Intersection-Fit & 706.2 & 27.3 & 3.8653\% \\
$n=72$ & 1 & Intersection & Inner-Outer & 195.5 & 14.6 & 7.4841\% \\
$n=72$ & 1 & Intersection & Group-Item & 195.5 & 14.6 & 7.4841\% \\
\midrule
$n=72$ & 3 & Group & Group-Blind & 319,645.9 & 2,044.7 & 0.6397\% \\
$n=72$ & 3 & Group & Intersection-Fit & 38,102.8 & 538.1 & 1.4121\% \\
$n=72$ & 3 & Group & Inner-Outer & 29,523.7 & 593.9 & 2.0116\% \\
$n=72$ & 3 & Group & Group-Item & 28,042.0 & 508.4 & 1.8130\% \\
$n=72$ & 3 & Intersection & Group-Blind & 2,382.1 & 107.3 & 4.5043\% \\
$n=72$ & 3 & Intersection & Intersection-Fit & 1,857.9 & 71.2 & 3.8324\% \\
$n=72$ & 3 & Intersection & Inner-Outer & 281.2 & 30.0 & 10.6836\% \\
$n=72$ & 3 & Intersection & Group-Item & 191.1 & 33.5 & 17.5491\% \\
\midrule
$n=72$ & 5 & Group & Group-Blind & 486,707.2 & 3,382.7 & 0.6950\% \\
$n=72$ & 5 & Group & Intersection-Fit & 47,176.0 & 719.0 & 1.5241\% \\
$n=72$ & 5 & Group & Inner-Outer & 32,945.7 & 726.7 & 2.2058\% \\
$n=72$ & 5 & Group & Group-Item & 28,027.6 & 767.6 & 2.7387\% \\
$n=72$ & 5 & Intersection & Group-Blind & 3,940.1 & 194.5 & 4.9367\% \\
$n=72$ & 5 & Intersection & Intersection-Fit & 2,741.5 & 93.3 & 3.4036\% \\
$n=72$ & 5 & Intersection & Inner-Outer & 621.1 & 49.7 & 8.0004\% \\
$n=72$ & 5 & Intersection & Group-Item & 132.6 & 34.8 & 26.2717\% \\
\bottomrule
\end{tabular}
}
\end{table*}

\begin{table*}[htbp]
\centering
\caption{Robustness Analysis: Multi-Seed Variance in BoardGameGeek ($T=10{,}000{,}000$)}
\label{tab:bgg_variance}
\resizebox{\textwidth}{!}{
\begin{tabular}{ccclrrr}
\toprule
\textbf{Catalog} & \textbf{$K$} & \textbf{Level} & \textbf{Algorithm} & \textbf{Avg Final $|\text{Regret}|$} & \textbf{Avg Std Dev} & \textbf{Relative Std} \\
\midrule
$n=15$ & 1 & Group & Group-Blind & 79,064.8 & 1,080.5 & 1.3666\% \\
$n=15$ & 1 & Group & Intersection-Fit & 72,146.8 & 532.8 & 0.7385\% \\
$n=15$ & 1 & Group & Inner-Outer & 36,677.2 & 599.5 & 1.6345\% \\
$n=15$ & 1 & Group & Group-Item & 36,677.2 & 599.5 & 1.6345\% \\
$n=15$ & 1 & Intersection & Group-Blind & 89,014.4 & 1,234.1 & 1.3864\% \\
$n=15$ & 1 & Intersection & Intersection-Fit & 36,947.7 & 481.4 & 1.3031\% \\
$n=15$ & 1 & Intersection & Inner-Outer & 61,128.1 & 465.6 & 0.7617\% \\
$n=15$ & 1 & Intersection & Group-Item & 61,128.1 & 465.6 & 0.7617\% \\
\midrule
$n=15$ & 3 & Group & Group-Blind & 151,523.5 & 1,063.6 & 0.7020\% \\
$n=15$ & 3 & Group & Intersection-Fit & 257,529.4 & 1,466.3 & 0.5694\% \\
$n=15$ & 3 & Group & Inner-Outer & 64,714.7 & 1,445.7 & 2.2340\% \\
$n=15$ & 3 & Group & Group-Item & 63,347.9 & 1,540.7 & 2.4322\% \\
$n=15$ & 3 & Intersection & Group-Blind & 155,000.2 & 1,012.5 & 0.6532\% \\
$n=15$ & 3 & Intersection & Intersection-Fit & 37,109.4 & 781.0 & 2.1045\% \\
$n=15$ & 3 & Intersection & Inner-Outer & 76,914.7 & 1,226.8 & 1.5950\% \\
$n=15$ & 3 & Intersection & Group-Item & 74,693.6 & 1,240.4 & 1.6606\% \\
\midrule
$n=15$ & 5 & Group & Group-Blind & 188,865.3 & 2,147.7 & 1.1372\% \\
$n=15$ & 5 & Group & Intersection-Fit & 351,137.7 & 1,587.4 & 0.4521\% \\
$n=15$ & 5 & Group & Inner-Outer & 83,368.0 & 1,404.8 & 1.6851\% \\
$n=15$ & 5 & Group & Group-Item & 82,216.5 & 1,471.9 & 1.7903\% \\
$n=15$ & 5 & Intersection & Group-Blind & 212,756.3 & 2,049.1 & 0.9631\% \\
$n=15$ & 5 & Intersection & Intersection-Fit & 33,921.2 & 1,027.2 & 3.0282\% \\
$n=15$ & 5 & Intersection & Inner-Outer & 79,980.5 & 946.4 & 1.1833\% \\
$n=15$ & 5 & Intersection & Group-Item & 74,247.5 & 845.5 & 1.1387\% \\
\midrule
$n=72$ & 1 & Group & Group-Blind & 341,354.4 & 1,171.9 & 0.3433\% \\
$n=72$ & 1 & Group & Intersection-Fit & 422,813.0 & 1,610.9 & 0.3810\% \\
$n=72$ & 1 & Group & Inner-Outer & 321,664.2 & 1,570.2 & 0.4881\% \\
$n=72$ & 1 & Group & Group-Item & 321,664.2 & 1,570.2 & 0.4881\% \\
$n=72$ & 1 & Intersection & Group-Blind & 306,410.8 & 1,018.3 & 0.3323\% \\
$n=72$ & 1 & Intersection & Intersection-Fit & 236,751.8 & 1,335.3 & 0.5640\% \\
$n=72$ & 1 & Intersection & Inner-Outer & 317,282.5 & 2,022.4 & 0.6374\% \\
$n=72$ & 1 & Intersection & Group-Item & 317,282.5 & 2,022.4 & 0.6374\% \\
\midrule
$n=72$ & 3 & Group & Group-Blind & 970,189.2 & 3,478.7 & 0.3586\% \\
$n=72$ & 3 & Group & Intersection-Fit & 2,149,024.8 & 3,362.6 & 0.1565\% \\
$n=72$ & 3 & Group & Inner-Outer & 314,186.0 & 4,907.2 & 1.5619\% \\
$n=72$ & 3 & Group & Group-Item & 319,181.6 & 4,197.1 & 1.3150\% \\
$n=72$ & 3 & Intersection & Group-Blind & 877,414.3 & 3,379.7 & 0.3852\% \\
$n=72$ & 3 & Intersection & Intersection-Fit & 297,348.8 & 2,239.1 & 0.7530\% \\
$n=72$ & 3 & Intersection & Inner-Outer & 520,033.7 & 2,859.9 & 0.5499\% \\
$n=72$ & 3 & Intersection & Group-Item & 517,766.4 & 2,720.3 & 0.5254\% \\
\midrule
$n=72$ & 5 & Group & Group-Blind & 1,587,617.8 & 3,227.2 & 0.2033\% \\
$n=72$ & 5 & Group & Intersection-Fit & 3,619,072.1 & 4,264.5 & 0.1178\% \\
$n=72$ & 5 & Group & Inner-Outer & 571,724.5 & 4,208.8 & 0.7362\% \\
$n=72$ & 5 & Group & Group-Item & 588,262.7 & 6,795.5 & 1.1552\% \\
$n=72$ & 5 & Intersection & Group-Blind & 1,340,602.7 & 3,156.7 & 0.2355\% \\
$n=72$ & 5 & Intersection & Intersection-Fit & 318,378.6 & 2,610.4 & 0.8199\% \\
$n=72$ & 5 & Intersection & Inner-Outer & 637,990.7 & 4,014.6 & 0.6293\% \\
$n=72$ & 5 & Intersection & Group-Item & 624,956.5 & 5,095.9 & 0.8154\% \\
\bottomrule
\end{tabular}
}
\end{table*}